%% file: camera_ready_arxiv.tex
\documentclass{alt2026}
\input{prel.tex}

\usetikzlibrary{arrows, decorations.markings}

\usetikzlibrary{trees, arrows.meta, positioning}

\newcommand{\myedge}[1]{(#1)}
\newcommand{\soadff}{\texttt{SOA-DFF}}
\newcommand{\asoadff}{\texttt{A-SOA-DFF}}
\newcommand{\truthvalues}{\{ \texttt{true},\texttt{false} \}}
\newcommand{\gaa}{\texttt{GAA}}
\newcommand{\hx}{\hat{x}}
\newcommand{\hy}{\hat{y}}
\newcommand{\X}{\mathcal{X}}

\newcommand{\bF}{\mathbb{F}}

\newcommand{\thist}{\hist_\bullet}

\title[Discriminative Feature Feedback with General Teacher Classes]{Discriminative Feature Feedback with General Teacher Classes}
\usepackage{times}

\altauthor{
 \Name{Omri {Bar Oz}} \Email{barozom@post.bgu.ac.il}\\
 \addr Faculty of Computer and Information Science, Ben-Gurion University of the Negev \\
 \AND
 \Name{Tosca Lechner} \Email{tosca.lechner@vectorinstitute.ai}\\
 \addr Vector Institute for Artificial Intelligence, Toronto, Canada\\
 \AND
 \Name{Sivan Sabato} \Email{sabatos@mcmaster.ca}\\
 \addr Department of Computing and Software, McMaster University\\ Canada CIFAR AI Chair, Vector Institute for Artificial Intelligence, Toronto, Canada\\ Faculty of Computer and Information Science, Ben-Gurion University of the Negev
}

\begin{document}
\maketitle

\begin{abstract}
  We study the theoretical properties of the interactive learning protocol \emph{Discriminative Feature Feedback} (DFF) \citep{DasguptaDeRoSa18}. The DFF learning protocol uses feedback in the form of \emph{discriminative feature explanations}. We provide the first systematic study of DFF in a general framework that is comparable to that of classical protocols such as supervised learning and online learning. We study the optimal mistake bound of DFF in the realizable and the non-realizable settings, and obtain novel structural results, as well as insights into the differences between Online Learning and settings with richer feedback such as DFF. We characterize the mistake bound in the realizable setting using a new notion of dimension. In the non-realizable setting, we provide a mistake upper bound and show that it cannot be improved in general. Our results show that unlike Online Learning, in DFF the realizable dimension is insufficient to characterize the optimal non-realizable mistake bound or the existence of no-regret algorithms.
\end{abstract}


\section{Introduction}

In this work, we study the theoretical properties of the interactive learning protocol \emph{Discriminative Feature Feedback} (DFF) \citep{DasguptaDeRoSa18}. Classical learning protocols, such as Supervised Learning \citep{VC71}, Online Learning \citep{Littlestone88} and Active Learning \citep{MccallumNi98}, involve interaction between the learner and the environment that is based only on examples and labels. However, learning in real-world environments can involve many other forms of interactions that provide rich feedback and can speed up learning. The theoretical study of protocols with rich feedback can shed light on their properties.

The DFF learning protocol uses feedback in the form of \emph{discriminative feature explanations}. 
In this setting, the learner provides an example as an explanation for each of its predictions, and the teacher provides the correct label and a feature explanation whenever the learner's prediction is incorrect. As an illustrative example, suppose the learner classifies a given image of an animal as a zebra, and provides as explanation an image of a previously seen zebra. Then, the teacher responds that the animal is actually a horse, and explains that unlike the example provided by the learner, the animal in the new image does not have stripes. Thus, the provided feature discriminates between the two examples and explains why they have different labels. This type of feedback has been successfully used in practical applications \citep[e.g.,][]{BWBSWPB10, ZCK15, LiangZoYu20}.

The theory of DFF has been studied so far \citep{DasguptaDeRoSa18,DasguptaSa20,Sabato23} under a specific component model that assumes that the domain of examples is covered by a small number of subsets of examples that have the same label, and that each pair of subsets with different labels can be discriminated by a single feature. Mistake upper bounds and lower bounds have been obtained for this model. However, DFF has not been studied so far in more generality. 

In this work, we provide the first systematic study of DFF in a general framework that allows any class of teachers. Here, a teacher defines both the true labeling function and the features that would be provided for each possible pair of a new example and a previous example that is provided as an explanation. Studying classes of teachers provides DFF with a framework comparable to that of classical protocols such as Supervised Learning and Online Learning. We study the optimal mistake bound of DFF in the realizable and the non-realizable settings, and obtain novel structural results, as well as insights into the difference between Online Learning and settings with richer feedback such as DFF.

\paragraph{Contributions} We define the notion of a general teacher class for DFF, and present a new relevant dimension of a teacher class, called \Tdim\ (\secref{realizable}). We show that this dimension characterizes the optimal mistake bound of a teacher class in the realizable setting (\thmref{tdim}), using a new Standard Optimal Algorithm for DFF. We demonstrate the use of the \Tdim\ by analyzing a less restrictive version of the component model of \cite{DasguptaDeRoSa18}. We then study the relationship between Online Learning and DFF (\secref{online}), by providing a two-way mapping between the two settings that preserves important properties (\thmref{otd-dto}, \thmref{ldimequaltdim}) and allows directly comparing the mistake bounds of pairs of mapped problems. We use this to show a strong separation between Online Learning and DFF: We show a teacher class with a \Tdim\ of $1$, while the Littlestone dimension of its Online Learning counterpart is infinite (\thmref{separation}).

Lastly, we study the \emph{non-realizable setting}. We provide a mistake upper bound (\thmref{upper}) via a simple standard algorithm that employs an optimal algorithm for the realizable setting as a subroutine. This algorithm is quite general, and provides a mistake upper bound for a wide range of natural interactive protocols (\thmref{informal}). We further show that this upper bound cannot be improved for DFF with general teacher classes (\thmref{lowerbound}).
To prove this lower bound, we use the idea of a \emph{secret-sharing scheme}, a well-known tool used in cryptography \citep{DBLP:journals/cacm/Shamir79}. We construct a teacher class using a set of such schemes that depend on the provided explanation.
From this lower bound, we conclude that there are no general no-regret algorithms for DFF problems with a finite DFF dimension. Nonetheless, we show that no-regret algorithms do exist for some DFF teacher classes, and conclude that unlike Online Learning, the optimal non-realizable mistake bound cannot be fully characterized using the realizable dimension. This raises an intriguing open question regarding the relationship between the properties of an interactive learning problem and its tolerance to teacher mistakes. 

Our analysis and results for the DFF setting provide a glimpse into the intricate landscape of learning using rich interactive learning protocols, which we hope will inspire further research. 

\section{Related work}

Interactive learning protocols are widely used in practical applications \citep[see, e.g., ][]{MosqueiraEtAl23}. For instance, \cite{teso2019explanatory} use  machine-generated explanations and corrections by users. Feature feedback was proposed as an aid for learning as early as \cite{CD90}. Similar ideas have been applied in various applications \citep[e.g.,][]{RMJ05,DMM08,S11a,MSCPY18}. Learning with explanations has also been studied for neural networks \citep{schramowski2020making} and Large Language Models \citep{lampinen-etal-2022-language}.

The formal DFF protocol was first defined in \citet{DasguptaDeRoSa18}. They further defined the component model and showed an algorithm and a mistake bound for this model in the realizable setting. Subsequently, \citet{DasguptaSa20} introduced a non-realizable version of DFF, in which the teacher might not always adhere to the protocol. They presented a robust algorithm for the component model, along with a mistake upper bound that depends on the number of rounds in which the teacher deviates from the protocol. An improved algorithm for the non-realizable component model was presented in \citet{Sabato23}.
Other types of feature feedback have been theoretically studied in \citet{PD17,VisotskyAtCh19}. More generally, \cite{HannekeKaMoVe22} study an interactive setting in which the learner can make arbitrary binary-valued queries, and \cite{YadavMoCh24} study auditing with explanations.

We note that the term ``teacher'' is also used in the Machine Teaching literature \citep[see, e.g.,][]{zilles2008teaching,doliwa2014recursive}. However, our use of this term is unrelated to the Machine Teaching paradigm. 

\section{Preliminaries}

For an integer $n$, denote $[n] = \{1,\ldots,n\}$.  
We assume a domain of examples $\cX$ and a finite domain of labels $\cY$. In addition, we consider Boolean \emph{features}, which are defined as functions $\phi:\cX \rightarrow \truthvalues$. For convenience, we sometimes use $1$ to mean $\true$ and $0$ to mean $\false$. Thus, we may say that $\phi(x)$ holds or that $\phi(x) = 1$. Denote the negation of $\phi$ by $\neg \phi \equiv 1-\phi$. We denote the set of available Boolean features for a given learning problem by $\Phi$. We generally assume that $\Phi$ is closed under negation.  
In various contexts, we use $\bot$ to represent a null element. We assume $\bot \notin \cX, \cY, \Phi$. We use `$\cdot$' to represent an arbitrary value. For instance, $(x,\cdot)$ represents a pair with $x$ as the first element and an arbitrary second element.

The Discriminative Feature Feedback (DFF) protocol \citep{DasguptaDeRoSa18} is defined as follows. At every round $t$,

\begin{itemize}
\item A new instance $x_t$ arrives.
\item The learning algorithm provides a predicted label $\hat{y}_t$, and an instance $\hat{x}_t$ that was previously observed with that label. This instance serves as the explanation for the predicted  label: ``$x_t$ is predicted to have label $\hat{y}_t$ because $\hat{x}_t$ was labeled $\hat{y}_t$''.
\item If the prediction is correct, no additional feedback is obtained from the teacher.
\item If the prediction is incorrect, the teacher provides the correct label of $x_t$, denoted $y_t$, and a feature $\phi \in \Phi$ that explains why $x_t$ does not have the same label as $\hat{x}_t$. $\phi$ is satisfied by $x_t$ but not by $\hat{x}_t$. 
\end{itemize}
It is further assumed that at least one fully labeled example is provided to the learner before the first round, thus providing at least one possible response for the first prediction round. We call the fully labeled examples given to the learner in advance the \emph{history}, usually denoted $\hist \subseteq \cX \times \cY$. 

\cite{DasguptaDeRoSa18} studied a specific \emph{component model} for teachers,
which assumes that the true teacher is consistent with some unknown cover of the domain. Each subset in the cover is called a \emph{component} and all the examples in a given subset have the same label. It is assumed that the teacher always provides a ``discriminative'' feature feedback $\phi_t$: a feature that is satisfied by all the examples in the component of $x$ and is not satisfied by all the examples in the component of $\hat{x}$. 
In this work, we study a more general version of DFF, in which we allow any \emph{teacher class}, an analog to the hypothesis classes of standard supervised learning analysis. The teacher class defines a given DFF problem. The component model is one such teacher class.

We first provide a general definition of a \emph{teacher}. In the realizable setting, the true teacher determines the feedback that would be provided to the learner in each round. Thus, the teacher determines both the labeling function and the feature feedback function in any possible interaction with the learner.

\begin{definition}[Teacher]
  A \emph{teacher} $T$ over $\cX, \cY, \Phi$ is a pair $(\ell, \psi)$, where $\ell: \cX \rightarrow \cY$ is a labeling function and $\psi: \cX \times \cX \rightarrow \Phi \cup \{\bot\}$ is a feature feedback function, such that for all $x, \hat{x} \in \cX$, if $\ell(x) \neq \ell(\hat{x})$ then $\phi := \psi(x,\hat{x}) \in \Phi$. In addition, $\phi$ satisfies $x$ and does not satisfy $\hat{x}$. 
\end{definition}

A \emph{teacher class} $\cT$ over $\cX,\cY, \Phi$ is then defined as a set of teachers over $\cX,\cY, \Phi$. We usually assume fixed $\cX, \cY, \Phi$, and omit the expression ``over $\cX, \cY, \Phi$'' when clear from context.

The number of mistakes made by a DFF algorithm $\cA$ in a specific run is the number of times its label prediction was incorrect. Given a teacher class $\cT$ and a history $\hist$, the worst-case mistake bound of $\cA$ over all possible input sequences, assuming feedback that is consistent with some teacher in $\cT$, is denoted $M(\cA, \cT, \hist)$. 

Under this framework, the component model of \cite{DasguptaDeRoSa18} is a teacher class $\cT_m$ for $m \in \nats$ that includes all teachers that are consistent with a component model that includes $m$ components. The analyses in \cite{DasguptaDeRoSa18,DasguptaSa20, Sabato23} provide bounds on $M(\cA, \cT_m, \hist)$ that depend on $m$ for specific algorithms $\cA$. In this work, we study the optimal mistake bound for general teacher classes.

\section{The optimal mistake bound for realizable DFF}\label{sec:realizable}

The VC-dimension \citep{VC71} maps each hypothesis class to an integer that characterizes the sample complexity of the optimal algorithm in the PAC setting. The Littlestone dimension \citep{Littlestone88} provides the optimal mistake bound achievable in the Online Learning setting, using a shattered tree construction as a witness. We define the Discriminative Feature Feedback Dimension (\Tdim) for a given teacher class, using a different tree construction, \emph{DFF tree} (\dfft), that aligns with the DFF protocol. We show that the dimension witnessed by a shattered \dfft\ characterizes the optimal mistake bound for DFF learning. A main difference between a Littlestone tree and the \dfft\ that we define is that in a Littlestone tree, the responses of the algorithm do no need to be directly encoded, since they do not affect the teacher's responses. This is not the case in DFF, since a different explanation by the algorithm may lead to a different feature response by the teacher. In particular, this means that unlike a Littlestone tree, in a \dfft\ multiple paths from the root to a leaf can be consistent with the same teacher.  

A \dfft\ (see \figref{shattered_dfft_tdim_3} for an illustration, and \figref{relaxed-component-shattered-dfft} below for a concrete example of a shattered DFFT induced by a specific teacher class.) is a rooted tree that represents options for interaction between the environment and the learner.
In a \dfft, each node
represents an action by the environment, which includes the feedback from the
teacher on the last prediction of the learner, as well as the next example to be presented. The root node specifies only the first example to be presented, without any feedback.
Each node is of the form $\node{y,\phi,x}$, where
$y \in \cY\cup \{\bot\}$, $\phi \in \Phi \cup\{\bot\}$ and
$x \in \cX \cup \{\bot\}$. Here, $y$ and $\phi$ represent the teacher feedback
for the latest prediction of the algorithm, and $x$ represents the next example. We use $y=\bot$ if this is the first round and there was no previous
prediction. We use $\phi = \bot$ if this is the first round, or if the last
prediction of the algorithm was correct. We use $x = \bot$ if this is the last round so no example comes next.

Each edge between a parent node and a child node is labeled by the prediction and the explanation example that the algorithm selected as a response to the example in the parent node. Edge labels are thus pairs $(\hat{x}, \hat{y}) \in \cX \times \cY$, where $\hat{y}$ is the prediction for the example in the parent node and $\hat{x}$ is the algorithm's explanation for this prediction.

  \begin{figure}[ht]
    \centering

    \resizebox{0.9\textwidth}{!}{
\begin{tikzpicture}[
  node distance=1.2cm and 2cm,
  every node/.style={draw, rounded corners, align=center},
  ->, >=Stealth
]

\node (x1) {$\node{\bot, \bot, x_1}$};

\node (x2) [below=of x1] {$\node{y_1, \phi_1, x_2}$};

\node (x3_0) [below left=of x2] {$\node{y^{(1)}_2, \phi^{(1)}_2, x^{(1)}_3}$};
\node (x3_1) [below right=of x2] {$\node{y^{(2)}_2, \phi^{(2)}_2, x^{(2)}_3}$};

\node (leaf1_0) [below left=of x3_0, xshift=1.5cm] {$\node{y^{(1)}_3, \phi^{(1)}_3, \bot}$};
\node (leaf2_0) [below=of x3_0] {$\node{y^{(2)}_3, \phi^{(2)}_3, \bot}$};
\node (leaf3_0) [below right=of x3_0, xshift=-1.5cm] {$\node{y^{(3)}_3, \phi^{(3)}_3, \bot}$};

\node (leaf1_1) [below left=of x3_1, xshift=1.5cm] {$\node{y^{(4)}_3, \phi^{(4)}_3, \bot}$};
\node (leaf2_1) [below=of x3_1] {$\node{y^{(5)}_3, \phi^{(5)}_3, \bot}$};
\node (leaf3_1) [below right=of x3_1, xshift=-1.5cm] {$\node{y^{(6)}_3, \phi^{(6)}_3, \bot}$};

\draw (x1) -- node[midway, right, draw=none] {$\myedge{x_0,y_0}$} (x2);

\draw (x2) -- node[midway, right, xshift=15pt, draw=none] {$\myedge{x_0,y_0}$} (x3_0);
\draw (x2) -- node[midway, right, xshift=10pt, draw=none] {$\myedge{x_1,y_1}$} (x3_1);

\draw (x3_0) to[out=210, in=60] node[midway, right,xshift=2.5pt, yshift = -5pt, draw=none] {$\myedge{x_0,y_0}$} (leaf1_0);
\draw (x3_0) -- node[midway, right,xshift=-2pt, draw=none] {$\myedge{x_1,y_1}$} (leaf2_0);

\draw (x3_0) to[out=-30, in=120] node[midway, right,xshift=2.5pt, yshift = 4pt, draw=none] {$\myedge{x_2,y^{(1)}_2}$} (leaf3_0);

\draw (x3_1) to[out=210, in=60] node[midway, right,xshift=2.5pt, yshift = -4pt, draw=none]{$\myedge{x_0,y_0}$} (leaf1_1);
\draw (x3_1) -- node[midway, right,xshift=-2pt, draw=none] {$\myedge{x_1,y_1}$} (leaf2_1);
\draw (x3_1)  to[out=-30, in=120] node[midway, right,xshift=7pt, yshift = 2pt, draw=none] {$\myedge{x_2,y^{(2)}_2}$} (leaf3_1);

\end{tikzpicture}
}
\caption{A shattered \dfft\ with history $H = \{(x_0, y_0)\}$ and height $3$.}
    \label{fig:shattered_dfft_tdim_3}

\end{figure}

\begin{definition}[DFF tree]\label{def:dfft} A \emph{DFF tree} (\dfft) over $\cX, \cY, \Phi$ is a rooted tree in which the following hold.
  \begin{enumerate}
    \item All nodes are of the form $\node{y,\phi,x}$ for $y \in \cY\cup \{\bot\}$, $\phi \in \Phi \cup\{\bot\}$ and
      $x \in \cX \cup \{\bot\}$.
    \item In the root node, $y = \phi = \bot$.
    \item A node has $x = \bot$ if and only if the node is a leaf.
    \item Each edge is labeled by a pair $(\hat{x}, \hat{y}) \in \cX \times \cY$.
          \item For every non-root node $\node{y,\phi, x}$ with an incoming edge $(\hx, \hy)$, if $y \neq \hy$ then $\phi \neq \bot$.
    
  \end{enumerate}
\end{definition}

A \dfft\ is not necessarily consistent with a given teacher class or history. 
A \emph{shattered \dfft}, which we define below, will allow us to define the DFF dimension of a given teacher class with a given history.
Denote parent and child nodes $v_1,v_2$ that are connected by an edge labeled $e$ by $v_1 \stackrel{e}{\longrightarrow} v_2$. 
We say that a path from the root to a node in a given \dfft\ is \emph{consistent} with a given teacher $T = (\ell,\psi)$ , if for every pair of parent and child nodes in the path of the form $\node{\cdot, \cdot, x} \stackrel{(\hx, \hy)}{\relbar\joinrel\relbar\joinrel\longrightarrow} \node{y, \phi, \cdot}$, we have $\ell(x) = y$ and if $y \neq \hat{y}$, $\psi(x,\hx) = \phi$. We call every such pair $(x,y)$ on a given path a \emph{labeled example} in the path. 

Since a DFF algorithm can only provide examples that were previously observed, the definition of a shattered \dfft\ takes into account the provided history, which is a non-empty set $\hist\subseteq \cX \times \cY$ of labeled examples.
We say that a teacher $T = (\ell, \cdot)$ is consistent with a history $\hist$ if for all $(x, y) \in \hist, \ell(x) = y$. We say that a teacher class $\cT$ is consistent with a history $\hist$ (and vice versa) if there is at least one teacher in $\cT$ which is consistent with $\hist$. Denote by $\cT_H$ the set of teachers that are consistent with $\hist$ in $\cT$. We can now define the notion of a shattered \dfft; See illustration in \figref{shattered_dfft_tdim_3}.

  \begin{definition}[Shattered \dfft]\label{def:sdfft}
    Let $\cT$ be a class of teachers over $\cX, \cY, \Phi$ and let $\hist \subseteq \cX \times \cY$ be a (non-empty) history that is consistent with $\cT$. 
    A \dfft\ over $\cX, \cY, \Phi$ is shattered by $\cT$ and $\hist$ if the following conditions hold.
    
    \begin{enumerate}

    \item \label{sdfft:label} For any node $\node{y, \phi, x}$ with an incoming edge $(\hx, \hy)$,      $y \neq \hy$. 
      \item \label{sdfft:edges} The outgoing edges of each non-leaf node $v$ are labeled by exactly all the pairs $(x,y)$ that satisfy at least one of the following conditions:
    \begin{itemize}
    \item $(x,y) \in \hist$, or
    \item $(x,y)$ is a labeled example in the path from the root of the tree to $v$.
    \end{itemize}
    
        \item \label{sdfft:consistent} Every path from the root to a leaf in the tree is consistent with at least one teacher from $\cT_\hist$.
        \item \label{sdfft:complete} The tree is complete; that is, all paths from the root to a leaf are of the same length.
    \end{enumerate}
  \end{definition}

The height of the tree is the number of edges in any path from the root to a leaf, which is also the number of examples in nodes along this path. We can now define the DFF dimension, \Tdim.

\begin{definition}[\Tdim]
  Let $\cT$ be a teacher class and let $\hist \subseteq \cX\times \cY$ be a (non-empty) history that is consistent with $\cT$. 
   $\Tdim(\cT, \hist)$ is the maximal integer $d$ such that there exists a \dfft\ of height $d$ that is shattered by $\cT$ and $\hist$.
\end{definition}

The following theorem shows that the \Tdim\ of a teacher class $\cT$ with a given history $\hist$ characterizes the optimal mistake bound for an algorithm for these $\cT$ and $\hist$. 
\begin{algorithm}[b]
\caption{Standard Optimal Algorithm for Discriminative Feature Feedback (\soadff)}
\begin{algorithmic}[1]
\Procedure{\soadff}{$\cT, \hist$}
    \State $V^1 \gets \cT$, $\histnum{1} \gets \hist$

    \For{$t=1,2,\dots$}
        \State Receive $x_t$
        \State $(\hat{x}_t, \hat{y}_t) \gets \argmin_{(\hat{x}, \hat{y}) \in \histnum{t}}\max_{y\neq \hat{y}, \phi \in \Phi} \Tdim(V^t|_{(x_t,\hat{x}, y, \phi)}, \histnum{t} \cup \{(x_t, y)\})$.
        \State Predict $\hat{y}_t$ and provide $\hat{x}_t$ as an explanation.
        \State Receive feedback $y_t$ and $\phi_t$.
        \State $V^{t+1} \gets V^t|_{(x_t,\hat{x}_t, y_t, \phi_t)}$, $\histnum{t+1} \gets \histnum{t} \cup \{(x_t, y_t)\}$
    \EndFor
\EndProcedure
\end{algorithmic}
\label{alg:soadff}
\end{algorithm}

  \begin{theorem}\label{thm:tdim}
    Let $\cT$ be a teacher class and let $\hist \subseteq \cX\times \cY$ be a (non-empty) history that is consistent with $\cT$. Then
    $\min_{\cA} M(\cA, \cT, \hist) = \Tdim(\cT, \hist)$, where the minimum is taken over all deterministic DFF algorithms. In particular, this minimum is attained by the DFF algorithm $\soadff$ given in \myalgref{soadff}.
  \end{theorem}

\soadff\ (\myalgref{soadff}), which attains the optimal realizable mistake bound, is a Standard Optimal Algorithm for DFF.
      Given a teacher class $\cT$, we denote its restriction based on a single round of DFF interaction by $\cT|_{(x,\hat{x}, y, \phi)} = \{T \in \cT \mid T = (\ell, \psi) \wedge \ell(x) = y \wedge \psi(x, \hat{x}) = \phi\}.$
Given a teacher class and a history, the algorithm selects in each round, out of its available responses $(\hx, \hy)$, the response that in the worst case, would reduce the DFF dimension the most.  The future DFF dimension is calculated with respect to the future history, which is updated based on the possible teacher responses. For convenience, we define $\Tdim(\emptyset, \hist) = -1$ for any history $\hist$. We prove that every time a mistake is made, the \Tdim\ of the remaining teachers is reduced by at least $1$. \thmref{tdim} is proved in \appref{tdimproof}.

    We note that the \Tdim\ also characterizes the optimal mistake bound only for deterministic algorithms. This is similar to the Littlestone dimension for Online Learning. In \cite{FilmusHaMeMo23}, a version of the Littlestone dimension for randomized algorithm is proposed. We defer a similar endeavor for DFF to future work.

\paragraph{Example: The \Tdim\ of a relaxed component model} To demonstrate the usefulness of the \Tdim, we analyze a relaxed version of the component model that was studied in \cite{DasguptaDeRoSa18}. 
In our setting, each component is defined via a conjunction of features that hold for all examples in the component, as well as a common label. Unlike the component model of  \cite{DasguptaDeRoSa18}, the relaxed version does not require the components to cover the domain. Examples that are not covered are assumed to share a ``default'' label. In addition, the relaxed model does not require the existence of a discriminative feature between pairs of components. In \appref{compex}, we give a formal definition of the model and show that if there are $R$ components with at most $M$ features in each conjunction, then the DFF dimension is at most $RM$, and this is tight for a natural maximal case. This result provides an optimal mistake bound for this problem. 

In the next section, we discuss the relationship between DFF and Online Learning.

  \section{Converting between DFF and Online Learning}\label{sec:online}

  What is the relationship between Online Learning and DFF? Intuitively, DFF is Online Learning with the potential for a reduced mistake bound due to feature feedback. We now formalize this correspondence and use it to derive a separation between the two settings. 
  First, we define a mapping from Online Learning problems to DFF problems and vice versa. An Online Learning problem can be converted to an equivalent DFF problem by providing unhelpful feature feedback. In the other direction, a DFF problem can be converted to an Online Learning problem, but the result may have a higher mistake bound. The relationship in both directions is shown by comparing the Littlestone dimension of the Online Learning problem to the DFF dimension of the DFF problem. We then show an example in which the Online Learning version has an infinite Littlestone dimension, while its DFF counterpart has a DFF dimension of $1$.

\newcommand{\first}{\cX}

We define the mapping $\otdmap$, which maps an Online Learning problem given by a hypothesis class $\cF \subseteq \cY^\cX$ to a DFF problem given by a teacher class $\cT$ and history $\hist$. The mapping defines explanations that do not divulge any additional information over the provided label. We further provide the converse mapping $\dtomap$, and show that this recovers the Online Learning problem back from a mapped DFF problem. Moreover, any DFF algorithm for the DFF problem $\otdmap(\cF)$ can be converted to an equivalent Online Learning algorithm that obtains the same number of mistakes, by simply ignoring the explanations of the teacher and refraining from providing explanation examples. We show that $\otdmap$ preserves the optimal mistake bound of the problem by comparing the respective dimensions.  For a set of labeled examples $S \subseteq \cY^\cX$, denote $S_\first := \{ x \mid \exists y \in \cY \st (x,y) \in S\}$. We two mappings are defined as follows.
  
  \begin{definition}[Mapping from Online Learning to DFF]\label{def:o2dmap}
    Given a hypothesis class $\cF \subseteq \cY^\cX$ defining an Online Learning problem, for any $y \in \cY$ let $\star_y$ denote a special example that is not in $\cX$. Let $\hist := \{(\star_y, y) \mid y \in \cY\}$, and define $\cX' := \cX \cup \hist_\first$. 
    Let the set of features over $\cX'$ be $\Phi := \{ \one[x] \mid x \in \cX' \}$, where $\one[x](x') = 1 \iff x = x'$. 

    For any $f \in \cF$, define $f': \cX' \rightarrow \cY$ by $f' := \{(x,f(x)) \mid x \in \cX\} \cup \{(\star_y, y) \mid y \in \cY\}$. 
    Let $\psi_f:\cX' \times \cX' \rightarrow \Phi \cup \{ \bot\}$ be a feature feedback function defined as $\psi_f(x,x') = \one[x]$. Let $T_f := (f', \psi_f)$ be a teacher over $\cX', \cY, \Phi$.     
    Define the mapping of $\cF$ from Online Learning to DFF by $\otdmap(\cF) := (\cT_\cF, \hist)$, where $\cT_\cF = \{T_f \mid f \in \cF\}$.
\end{definition}

The converse mapping, $\dtomap$, is defined next.

\begin{definition}[Mapping from DFF to Online Learning]\label{def:d2omap}
  Given a DFF problem defined by a teacher class $\cT$ over $\cX, \cY, \Phi$ and a history $\hist \subseteq \cX \times \cY$, let $\cX' = \cX \setminus \hist_\first$. For a teacher $T = (\ell, \psi) \in \cT$, let $f_T:= \ell |_{\cX'}$. The mapping of the DFF problem $(\cT, \hist)$ to Online Learning is $\dtomap(\cT, \hist) = \cF$, where $\cF = \{ f_T \mid T \in \cT\}$. 

\end{definition}

First, we show that mapping from Online Learning to DFF and back recovers the original online learning problem, thus proving that no information is lost in the mapping. The proof is provided in \appref{online}.
\begin{theorem}\label{thm:otd-dto} For any hypothesis class $\cF \in \cY^\cX$, we have $\dtomap(\otdmap(\cF)) = \cF$.
\end{theorem}

The next theorem shows that the Littlestone dimension of an Online Learning problem is equal to the DFF dimension of its mapping to a DFF problem. In other words, the optimal mistake bound of the problem is preserved by the conversion. 
Denote the Littlestone dimension of an Online Learning problem $\cF$ by $\Ldim(\cF)$. Note that since our problems can have more than two labels, we use the multiclass Littlestone dimension, which is also witnessed by a binary tree \citep{DanielySaBeSh15}. The theorem is proved in \appref{online}. We give here a partial sketch.

\begin{theorem}\label{thm:ldimequaltdim}
  For any hypothesis class $\cF \in \cY^\cX$, $\Tdim(\otdmap(\cF)) = \Ldim(\cF)$.
\end{theorem}
\begin{proof}[Sketch for the lower bound]
  Given a shattered Littlestone tree for $\cF$ with height $\Ldim(\cF)$, we construct $\tree'$, a shattered \dfft\ of the same height, by inductively mapping from nodes in $\tree$ to nodes in $\tree'$. For a node $v$ in $\tree'$, $N(v)$ denotes that the node in $\tree$ that it was mapped from. 
  Let $y_1,y_2$ be the two outgoing edges from $N(v)$ in $\tree$, and let $x_1,x_2$ be the respective examples in the target nodes $u_1,u_2$ of these edges in $\tree$ (see \figref{ldimtdim}). In $\tree'$, we set the target of each outgoing edge $(\hat{x},\hat{y})$ to a node $\node{y_1, \one[x], x_1}$ if $\hat{y} \neq y_1$, and to $\node{y_2, \one[x], x_2}$ if $\hat{y} = y_1$. We prove that this inductive construction results in a shattered \dfft.
  \end{proof}
 \begin{figure}[h]
   \begin{center}
     \resizebox{0.9\textwidth}{!}{
       \tikzstyle{vecArrow} = [thick, decoration={markings,mark=at position
   1 with {\arrow[semithick]{open triangle 60}}},
   double distance=1.4pt, shorten >= 5.5pt,
   preaction = {decorate},
   postaction = {draw,line width=1.4pt, white,shorten >= 4.5pt}]
   
\begin{tikzpicture}[>=stealth, node distance=2cm]
  \node (v) [draw, circle] {$x$};
  \node[below=2pt of v] {$N(v)$};
  \node (v1) [below left=of v, draw, circle] {$x_1$};
  \node[below=2pt of v1] {$u_1 = N(v')$};
  \node (v2) [below right=of v, draw, circle] {$x_2$};
\node[below=2pt of v2] {$u_2 = N(v'')$};
  \draw[->] (v) -- (v1) node[midway, left] {$y_1$};
  \draw[->] (v) -- (v2) node[midway, right, xshift=2pt] {$y_2$};

  \coordinate (arrowStart) at ($(v)!0.5!(v1) + (3.2,01)$);
  \draw[vecArrow] (arrowStart) -- ++(2,0);
  \node (va) [draw, ellipse, inner sep = 2pt, right=19em of v] {$\node{\cdot, \cdot, x}$};
  \draw[->, dashed] (va) -- ++(1,-1.7);
  \draw[->, dashed] (va) -- ++(-1,-1.7);
  \node [below=1cm of va] {$\ldots\ldots$};
  \node[below=2pt of va] {$v$};
  \node (va1) [below left=of va, draw, ellipse, inner sep=2pt] {$\node{y_1, \one[x], x_1}$};
  \node[below=2pt of va1] {};
  \node (va2) [below right=of va, draw, ellipse, inner sep=2pt] {$\node{y_2, \one[x], x_2}$};
  \draw[->] (va) -- (va1) node[midway, left, xshift = -3pt] {$(\hat{x},\hat{y})$, $\hat{y} \neq y_1$};
  \draw[->] (va) -- (va2) node[midway, right, xshift=2pt] {$(\hat{x},y_1)$};
  \node[below=2pt of va1] {$v'$};
  \node[below=2pt of va2] {$v''$};
\end{tikzpicture}
}
\caption{Illustrating the construction of $\tree'$ in the proof of \thmref{ldimequaltdim}}
\label{fig:ldimtdim}
\end{center}
\end{figure}

Using the mappings defined above, we can compare the optimal mistake bound of a DFF problem with the one of its Online Learning equivalent, in which no explanations are provided. A strong separation between the two settings is demonstrated by the following example, in which the \Tdim\ is 1, while the Littlestone Dimension is infinite. Moreover, this example has an infinite Littlestone tree, implying that an adversary can force the algorithm to make a mistake in every round in an infinite sequence of rounds \citep[see][]{bousquet2021theory}. 
In \cite{DasguptaDeRoSa18}, the power of DFF compared to standard Online Learning was discussed in the context of computational complexity and CNF formulas. The example and theorem below provide a clear statistical separation.

We use the following \emph{natural feature construction}: Let $\cX= \{0,1\}^\nats$, and define the Boolean function $f_n$ by $f_n(x) := x(n)$. Set $\Phi := \cup_{n \in \nats} \{ f_n, \neg f_n\}$. F`or vectors in $\{0,1\}^\nats$, we use an overline to indicate infinite repetition, so that $\bar{0},\bar{1}$ are the all-zero and all-one vectors, and $\overline{01}$ is a vector of alternating coordinate values. 

\begin{example}\label{exm:1-tdim-high-ldim}
  Let $\cY = \{0,1\}$ and let $\cX,\Phi$ be the natural feature construction defined above. For any $n \in \nats$, define a teacher $T_n = (f_n, \psi_n)$, where
  \[
    \psi_n(x,x') :=
    \begin{cases}
      f_n  & f_n(x) = 1, f_n(x') = 0\\
      \neg f_n  & f_n(x) = 0, f_n(x') = 1\\
      \bot & \text{ otherwise }
    \end{cases}
  \]
  Let the teacher class be $\cT = \{ T_n \mid n \in \nats\}$. 
  Define $\hist = \{(\bar{0}, 0), (\bar{1}, 1)\}$.
\end{example}

\begin{theorem}\label{thm:separation}
  For $\cT, \hist$ as defined in \exmref{1-tdim-high-ldim}, Then $\Tdim(\cT,\hist) = 1$ while $\Ldim(\dtomap(\cT,\hist)) = \infty$. Moreover, there exists an infinite Littlestone tree for $\dtomap(\cT,\hist)$. 
\end{theorem}
\begin{proof}
  Let $\cX' = \cX \setminus \{\bar{0},\bar{1}\}$. Let $\cF := \dtomap(\cT, \hist)$. Then $\cF = \{ f_n|_{\cX'} \mid n \in \nats\}$. To see that $\Ldim(\cF) = \infty$, consider a complete binary Littlestone tree constructed as follows: Set the root node to $x = \overline{01}$. The rest of the tree is constructed inductively, as follows. For a given node $v$, let $(x_1,y_1),\ldots,(x_t,y_t)$ be the sequence of labeled examples in the path from the root to $v$. Let $A_v := \{ n \mid \forall i \in [t], x_i(n) = y_i\}$. Set the example in node $v$ to some $x$ such that $x(A_v) = \overline{01}$. It can be seen by induction that for each $v$, $|A_v| = \infty$. Therefore, for any path from the root to a leaf $v$, any $f_n$ for $n \in A_v$ is consistent with the labeled examples on the path. It follows that there exists a Littlestone tree of any finite size for $\cF$, as well as a Littlestone tree of infinite depth.

  To show that $\Tdim(\cT,\hist) = 1$, it suffices to provide a deterministic algorithm that makes at most a single mistake, since clearly no algorithm for this problem has a mistake upper bound of $0$. Consider the algorithm that always predicts the label $0$ with the explanation $\bar{0}$, until it makes the first mistake on some example. The true label of the example is $1$, and the feature feedback provided by the teacher for this example is the true labeling function. Thereafter, the algorithm predicts using the true labeling function and does not make any more mistakes. This completes the proof.
  \end{proof}

  \section{DFF in the non-realizable setting}\label{sec:agnostic}
  The previous sections discuss DFF in the realizable setting, in which the feedback provided to the algorithm is assumed to be consistent with one of the teachers in the teacher class. Previous works \citep{DasguptaSa20, Sabato23} studied a non-realizable setting for the original component model of \cite{DasguptaDeRoSa18}, and provided mistake upper bounds. In the non-realizable setting, it is assumed that the feedback provided to the algorithm is consistent with one of the teachers in the teacher class, except for at most $k \in \nats$ rounds, in which the feedback might have a different label and/or a different feature feedback than prescribed by the teacher. This behavior is called an \emph{exception} from the protocol. Rounds with an exception are called \emph{exception rounds}. We call the setting with up to $k$ exception rounds the \emph{$k$-non-realizable} setting.
  Here, we provide new bounds for general teacher classes for this setting.

  In Online Learning, the $k$-non-realizable setting is equivalent to assuming that at most $k$ of the labels provided by the environment are different from those of the true labeling function. For an Online Learning problem with a hypothesis class $\cF$, the optimal mistake bound assuming at most $k$ label errors is $k + \Theta\left(\sqrt{k\cdot \Ldim(\cF)} + \Ldim(\cF)\right)$, where the upper bound is obtained by a randomized algorithm (\citealt{FilmusHaMeMo23}, following \citealt{BenDavidPaSh09,AlonBeDaMoNaYo21}). 
  
  The apparent similarity of DFF to Online Learning, as discussed in
  \secref{online}, may lead to the expectation that the dependence on $k$ in
  mistake bounds for DFF would be similar to that of Online Learning. However,
  we show that this is not the case. We obtain a mistake upper bound of \mbox{$O(k\cdot\Tdim(\cT,\hist))$}, using a simple
  deterministic algorithm for the $k$-non-realizable setting. We then show
  that this upper bound cannot be improved for general DFF problems, even with a randomized algorithm, assuming a mild version of adversarial adaptivity that holds also for the Online Learning bound mentioned above. We conclude that unlike Online Learning, a no-regret algorithm for general
  DFF problems with a finite dimension does not exist in this setting. Moreover, again unlike Online Learning, the dependence of the optimal mistake bound on $k$ is highly dependent on the DFF problem, raising a new open question regarding the factors that control this dependence.

  The mistake upper bound for the $k$-realizable setting is obtained using the simple Agnostic Standard Optimal Algorithm \asoadff\ listed in \myalgref{asoadff}. This algorithm runs the \soadff\ algorithm that is meant for the realizable setting, and restarts it whenever it makes more mistakes than the realizable mistake upper bound. Denote by $M^L_k(\cA, \cT, \hist)$ the worst-case number of mistakes made by a DFF algorithm $\cA$ for $\cT,\hist$ on a sequence of $L$ rounds in the $k$-non-realizable setting.

\begin{algorithm}[ht]
\caption{Agnostic Standard Optimal Algorithm for DFF (\asoadff)}
\begin{algorithmic}[1]
  \Procedure{\asoadff}{$\cT, \hist$}
  \State Initialize $\soadff(\cT,\hist)$ and set $M \leftarrow 0$.
    \For{$t=1,2,\dots$}
    \State Receive $x_t$ and provide it to $\soadff$; Get output $(\hat{x}_t, \hat{y}_t)$ from $\soadff$.
    \State Predict $\hat{y}_t$ and provide $\hat{x}_t$ as an explanation.
    \State Receive feedback $y_t, \phi_t$ and provide it to $\soadff$.
    \State If $\hat{y}_t \neq y_t$, $M \leftarrow M+ 1$.
    \State If $M = \Tdim(\cT, \hist) + 1$, re-initialize $\soadff(\cT,\hist)$ and set $M \leftarrow 0$
    \EndFor
\EndProcedure
\end{algorithmic}
\label{alg:asoadff}
\end{algorithm}

  \begin{theorem}\label{thm:upper}
    For any teacher class $\cT$ and history $\hist$, and for any $L \in \nats$,
     if $\Tdim(\cT,\hist) = d$ then 
    \[
      M^L_k(\asoadff,\cT,\hist) \leq (k + 1) d + k.
      \]
\end{theorem}
\begin{proof}
Call each sub-sequence of the run between initializations of $\soadff$ a ``run segment''. In each run segment, $\soadff$ makes $d+1$ mistakes. Let $N$ be the total number of run segments. Then the total number of prediction mistakes made by the algorithm is at most $N(d+1) + d$. Since the realizable mistake bound of \soadff\ is $M(\soadff,\cT,\hist) = d$, all teachers in $\cT_\hist$ are inconsistent with all of the $N$ run segments. Therefore, any such teacher is inconsistent with at least $N$ rounds in the whole run. If there are at most $k$ exceptions during the run, it follows that $N \leq k$. Thus, $M^L_k(\asoadff,\cT,\hist) \leq k(d+1) + d = (k+1)d+k$.
  \end{proof}

  We note that the upper bound in fact holds for a wide range of interactive prediction protocols, of which DFF is but one example. We provide a formal statement and proof of the following result in \appref{interactive}.
  \begin{theorem}[Informally]\label{thm:informal}
    For any ``natural'' interactive protocol with a realizable mistake upper bound of $d$, there is an algorithm that obtains a mistake upper bound of at most $(k+1)d+k$ for any run with at most $k$ exception rounds.    
  \end{theorem}

We further show that the upper bound of  \thmref{upper} is tight, by proving a nearly matching lower bound. The lower bound holds also for randomized algorithms with a mildly adaptive adversary. In general online prediction problems, an oblivious adversary fixes the sequence of loss functions in advance, while an adaptive adversary decides on the loss function in each round based on the algorithm's actions in the previous rounds \citep{AroraDeTe12}. However, in Online Learning, both of these adversaries obtain the same optimal mistake bound of $k + \Theta\left(\sqrt{k\cdot \Ldim(\cF)} + \Ldim(\cF)\right)$, using the same randomized Weighted Majority algorithm of \cite{LittlestoneWa94}. In the lower bound below, we consider a DFF version of the adaptive adversary, which fixes the label and the feature feedback function ahead of  each round based on the algorithm's actions in previous rounds. The effect of other types of adversaries on the optimal mistake bound of DFF is an intriguing question that we leave for future work. Below, $M_k^L(\cdot,\cdot,\cdot)$ denotes the \emph{expected} mistake bound of a randomized algorithm, under a worst-case adaptive adversary.

\begin{theorem}\label{thm:lowerbound}
  Let $\cX = \nats$, $\cY = \{0,1\}$, $\Phi = \{0,1\}^\cX$,
  $\thist = \{(2, 0), (1,1)\}$. Let $d \geq 2$. There is a teacher class $\cT^d$ such that $\Tdim(\cT^d,\thist) \leq d$ and 
for any $k \geq 1$ and $L \geq 4(k+1)d$, for any algorithm $\cA$ (including randomized algorithms) 
  \[
    M^L_k(\cA, \cT^d, \thist) \geq (k+1)d - k - 2.
  \]

\end{theorem}
Taken together with \thmref{upper}, we conclude that the optimal mistake bound for a general teacher class with a DFF dimension of $d$ in the $k$-non-realizable setting is $(k+1)d \pm \Theta(k)$. 
In addition, this implies that no-regret algorithms for the general DFF problem do not exist for $d \geq 3$: for $k = L/(4d)-1$, we obtain a regret bound of $M_k^L(\cA,\cT^d,\thist) -k \geq (k+1)d-2k-2 = \Omega(L)$.

While \thmref{lowerbound} shows that the mistake upper bound of \thmref{upper} is tight, it does not show a lower bound that holds for all DFF problems. We observe that unlike Online learning, in DFF the dimension that characterizes the realizable mistake bound is insufficient to characterize the mistake bound in the non-realizable setting. To demonstrate this, consider an Online Learning problem $\cF \subseteq \cY^\cX$ with $\Ldim(\cF) = d$ and its DFF counterpart $(\cT, \hist) := \otdmap(\cF)$. By \thmref{ldimequaltdim}, $\Tdim(\cT,\hist) = d$. The optimal mistake bound for the DFF problem is the same as that of is Online Learning counterpart, $k + \Theta(\sqrt{kd} + d) \ll \Theta(kd)$, implying also the existence of a no-regret algorithm. As a different example, consider the component model of \citet{DasguptaDeRoSa18}. They show that the realizable mistake bound for this model with $m$ components is $O(m^2)$, and this is shown to be tight in \cite{DasguptaSa20}. It follows that the \Tdim\ of this problem is $d = \Theta(m^2)$. On the other hand, there exists an algorithm for this problem in the $k$-non-realizable setting with a mistake bound of $O(m^2 + mk) = O(d + \sqrt{d}k)$ \citep{Sabato23}, giving yet a different dependence on $k$. However, no matching lower bound has been shown.   
These examples raise an interesting open question: which properties of a DFF problem determine the dependence of its optimal $k$-non-realizable mistake bound on $k$.

The proof of \thmref{lowerbound}, provided in \appref{prooflowerbound}, employs a teacher class that satisfies the following: If the algorithm provides the same $1$-labeled example as an explanation on $d$ different rounds, for examples whose true label is $0$, then it obtains $d$ feature feedbacks, that together perfectly identify the true labeling function. However, it is impossible to obtain \emph{any} information on the true labeling function (beyond the information provided by the label feedback) with fewer than $d$ feature feedbacks that were provided for the same explanation. Moreover, it is impossible to obtain such information by combining feature feedbacks provided for different explanations.
In addition, the environment can provide feedback in a way that makes it impossible to distinguish between exception and non-exception rounds.
In this problem, the environment can cause the algorithm to repeatedly use examples provided in exception rounds as explanations. These examples are actually labeled $0$, but are provided to the algorithm with label $1$ in an exception round, thus the provided feature feedback when they are used as explanations is uninformative. However, the algorithm cannot identify this before $d$ mistakes  have already been made for each of these $k$ explanations. 

A core property of the construction above is the ability to perfectly describe the true function using $d$ items of information, while providing no information if there are only $d-1$ items. We achieve this via secret sharing \citep{DBLP:journals/cacm/Shamir79}, an idea commonly used in cryptography. A secret-sharing scheme allows representing a secret using a number of parts, such that only if a sufficient number of the parts is known, the secret is revealed, and no information about the secret is revealed from fewer parts.

More specifically, we employ a \emph{$(d,n)$-threshold secret sharing scheme} \citep{DBLP:journals/cacm/Shamir79}. This is a mapping from a secret $s$ to a set of $n$ partial secrets $\bar{s}$ such that having access to any subset of $\bar{s}$ of size at least $d$ allows one to perfectly reconstruct $s$, and having access to only a subset of $\bar{s}$ of size $d-1$ or less does not reveal \emph{any} information about the secret $s$. This scheme is based on Langrange's interpolation theorem \citep{Lagrange1795}. The latter states that any polynomial $P$ of degree $d-1$ can be interpolated when having access to $d$ distinct evaluations $((x_1, P(x_1)), \dots, (x_d, P(x_d))))$ of the polynomial. Moreover, for any $d$ distinct evaluations, there exists a polynomial $P$ that satisfies them.
The full scheme is specified in \appref{prooflowerbound}.
We use this to construct the class $\cT^d$ for \thmref{lowerbound}, in which each possible explanation provides parts of a different secret-sharing scheme. The construction and the full proof of \thmref{lowerbound} are provided in \appref{prooflowerbound}.

    \section{Conclusion}

    In this work, we study Discriminative Feature Feedback in a general framework based on teacher classes, and show new results characterizing its mistake bound in the realizable setting. We further showed separation between DFF and Online Learning in the realizable case. Lastly, we showed that unlike Online Learning, in DFF the optimal mistake bounds in the non-realizable setting exhibit a rich behavior that depends on the specific teacher class. This presents an intriguing open question regarding the interplay between rich feedback and robustness to deviations from the protocol, as well as the challenge of characterizing this behavior based on properties of the learning problem.
        
\acks{We thank our colleagues and funding agencies;
We acknowledge the support of the Natural Sciences and Engineering Research Council of Canada (NSERC) [funding reference number RGPIN-2024-05907]. Resources used in preparing this research were provided, in part, by the Province of Ontario, the Government of Canada through CIFAR, and companies sponsoring the Vector Institute; see \url{https://vectorinstitute.ai/partnerships/current-partners/}. This work was supported by the Israeli Council for Higher Education (CHE) via the Data Science Research Center at Ben-Gurion University.
}    

\bibliography{references,mybib,shared}

\appendix

\section{Proof of \thmref{tdim}:  optimal realizable mistake bound for DFF}\label{app:tdimproof}

  \thmref{tdim} is proved via the following two lemmas.

  \begin{lemma} \label{lem:soa-mistakes}
    Let $\cT,\cH$ as in \thmref{tdim}. 
    There exists a DFF algorithm $\soadff$ such that\\ \mbox{$M(\soadff, \cT, \hist) \leq \Tdim(\cT, \cH)$}. 
  \end{lemma}

  \begin{lemma} \label{lem:at-least}
    Let $\cT,\cH$ as in \thmref{tdim}. 
    For any deterministic DFF algorithm $\cA$, $M(\cA, \cT, \hist) \geq \Tdim(\cT, \hist)$.
  \end{lemma}

  To prove \lemref{soa-mistakes}, we first prove a structural property of the DFF dimension.

\begin{lemma}\label{lem:tdim-hist-and-teachers}
    Let $\cT_1, \cT_2$ be teacher classes such that $\cT_1 \subseteq \cT_2$. Let $\hist_1, \hist_2 \subseteq \cX \times \cY$ be (non-empty) histories such that $\hist_2 \subseteq \hist_1$. Then $\Tdim(\cT_1, \hist_1) \leq \Tdim(\cT_2, \hist_2)$.
\end{lemma}

\begin{proof}
  Let $\treenum{1}$ be a shattered \dfft\ for $\cT_1, \hist_1$ of height $\Tdim(\cT_1, \hist_1)$. We construct a shattered \dfft\ for $\cT_2, \hist_2$ of the same height, thus proving that $\Tdim(\cT_2, \hist_2) \geq \Tdim(\cT_1, \hist_1)$.

  $\treenum{2}$ is obtained by removing from $\treenum{1}$ all edges (and the sub-trees descended from them) that are labeled by $(\hx,\hy) \in \hist_1 \setminus \hist_2$, such that $(\hx,\hy)$ is not a labeled example in the path from the root of the tree to the source node of the edge. 
  We show that $\treenum{2}$ is a shattered \dfft\ for $\cT_2, \hist_2$ and that it has the same height as $\treenum{1}$.

  It is easy to see that property \ref{sdfft:label} (different labels) holds for $\treenum{2}$, since it is a sub-graph of $\treenum{1}$. 
Property \ref{sdfft:edges} (outgoing edges) follows by the fact that all remaining outgoing edges of each node satisfy one of the conditions of this property, since $\hist_2 \subseteq \hist_1$. In addition, all required edges exist in $\treenum{2}$, since they all exist in $\treenum{1}$ and are not removed.
Property \ref{sdfft:consistent} (consistency with a teacher in $(\cT_2)_{\hist_2}$) follows from the same property in $\treenum{1}$ for $\cT_1,\hist_1$, since every path from the root to a leaf in $\treenum{2}$ is also such a path in $\treenum{1}$, and since $\cT_2 \supseteq \cT_1$ and $\hist_2 \subseteq \hist_1$.

Lastly, for property \ref{sdfft:complete} (tree completeness), we show that since $\treenum{1}$ is complete, so is $\treenum{2}$. Observe that by property \ref{sdfft:edges} of $\treenum{1}$ and since $\emptyset \neq \hist_2 \subseteq \hist_1$, every node in $\treenum{1}$ has at least one outgoing edge $(\hx,\hy) \in \hist_2$. Therefore, this node is a non-leaf in $\treenum{2}$ as well. Thus, all remaining paths from root to leaf in $\treenum{2}$ are of the same length as they have in $\treenum{1}$. It follows that $\treenum{2}$ is also complete. This also implies that $\treenum{2}$ has the same height as $\treenum{1}$.

To conclude, $\treenum{2}$ is a shattered \dfft\ for $\cT_2, \hist_2$ and has height $\Tdim(\cT_1, \hist_1)$. 
Therefore, $\Tdim(\cT_2, \hist_2) \geq \Tdim(\cT_1, \hist_1)$.
\end{proof}

We now prove the mistake bound of \soadff.
\begin{proof}[Proof of \lemref{soa-mistakes}]
    We show that for every round $t$ in which \soadff\ makes a mistake,
  \begin{equation}\label{eq:strict}
    \Tdim(V^{t+1}, \histnum{t+1}) < \Tdim(V^t,\histnum{t}).
  \end{equation}
  It follows that \soadff\ makes at most $\Tdim(V^1,\histnum{1}) = \Tdim(\cT, \hist)$ mistakes. 

  By \myalgref{soadff}, we have
  $V^{t+1} = V^t|_{(x_t,\hat{x}_t,y_t,\phi_t)}$ and $\histnum{t+1} = \histnum{t} \cup \{(x_t, y_t)\}$.
Thus, by \lemref{tdim-hist-and-teachers}, $\Tdim(V^{t+1}, \histnum{t+1}) \leq \Tdim(V^t,\histnum{t})$. It therefore suffices to show that the inequality is strict.
  Let $t$ be a round in which \soadff\ makes a mistake. Assume for the sake of contradiction that 
    $\Tdim(V^{t+1}, \histnum{t+1}) = \Tdim(V^t, \histnum{t}).$

      By \myalgref{soadff},
      \[
        (\hat{x}_t, \hat{y}_t) = \argmin_{(\hat{x},\hat{y}) \in \histnum{t}} \max_{y \neq \hat{y}, \phi \in \Phi} \Tdim(V^t|_{(x_t,\hat{x},y,\phi)}, \histnum{t} \cup \{(x_t, y)\}).
          \]

    Let $\histnum{t} = \{(\tilde{x}_1,\tilde{y}_1),\ldots,(\tilde{x}_n,\tilde{y}_n)\}$.  Then, for every $i \in [n]$ there exist $\overline{y}_i \in \cY$ and $\phi_i \in \Phi$ such that $\overline{y}_i \neq \tilde{y}_i$ and, denoting  
    $V^t_i := V^t|_{(x_t,\tilde{x}_i, \overline{y}_i, \phi_i)}, \histnumi{t} := \histnum{t} \cup \{(x_t, \overline{y}_i)\}$, it holds that
    \[
    \Tdim(V^{t+1}, \histnum{t+1}) \leq \Tdim(V^t_i, \histnumi{t}) \leq \Tdim(V^t, \histnum{t}).
  \]
  The first inequality follows from the choice of $(\hx_t, \hy_t)$, and the right inequality follows from \lemref{tdim-hist-and-teachers}.
  By the assumption, the RHS and the LHS above are equal. Thus, 
  \[
    \Tdim(V^{t+1}, \histnum{t+1}) = \Tdim(V^t_i, \histnumi{t}).
  \]

  For $i \in [n]$, let $\treenum{i}$ be a shattered \dfft\ for $V^{t}_i$ and $\histnumi{t}$, with height $\Tdim(V^{t}_i, \histnumi{t}) = \Tdim(V^t, \histnum{t})$. We construct a shattered \dfft\ for $V^t$ with history $\histnum{t}$ with height $\Tdim(V^t, \histnum{t}) + 1$, thus contradicting the definition of DFF dimension.

  The shattered \dfft, denoted $\tree$, is constructed as follows (see \figref{soa-figure} for an illustration).
  The root of $\tree$ is set to $\node{\bot, \bot, x_t}$. Its outgoing edges are set to edges labeled by $(\tilde{x}_i, \tilde{y}_i)$ for all $i \in [n]$. The target node of the edge labeled $(\tilde{x}_i, \tilde{y}_i)$ is set to $v_i := \node{\overline{y}_i, \phi_i, x_r^i}$, where $\overline{y}_i, \phi_i$ are as defined above and $x_r^i$ is the example at the root of $\treenum{i}$. The subtree descending from $v_i$ is set to be identical to the subtree descending from the root of $\treenum{i}$.

\begin{figure}[ht]
    \centering

    \resizebox{0.5\textwidth}{!}{
\begin{tikzpicture}[
  node distance=1.8cm and 2cm,
  every node/.style={draw, rounded corners, align=center},
  ->, >=Stealth
]

\node (x0) {$\node{\bot, \bot, x_t}$};

\node (x1) [below left=of x0, xshift=0.75cm] {$\node{\overline{y}_1, \phi_1, x^1_r}$};

\node (xi) [below=of x0] {$\node{\overline{y}_i, \phi_i, x^i_r}$};
\node (xn) [below right=of x0, xshift=-0.75cm] {$\node{\overline{y}_n, \phi_n, x^n_r}$};

\node[draw=none, inner sep=0] at ($(x1)!0.985!(xi) + (-1.5cm,0)$) {$\ldots$};
\node[draw=none, inner sep=0] at ($(xi)!0.95!(xn) + (-1.5cm,0)$) {$\ldots$};

\usetikzlibrary{shapes.geometric, positioning}

\node[
    regular polygon,
    regular polygon sides=3,
    minimum size=6mm,
    align=center,
    shape border rotate=0,
    below=-0.09cm of x1
] (tr1) {\scalebox{1}{$\treenum{1}$}};

\node[
    regular polygon,
    regular polygon sides=3,
    minimum size=6mm,
    align=center,
    shape border rotate=0,
    below=-0.09cm of xi
] (tri) {\scalebox{1.1}{$\treenum{i}$}};

\node[
    regular polygon,
    regular polygon sides=3,
    minimum size=6mm,
    align=center,
    shape border rotate=0,
    below=-0.09cm of xn
] (trn) {\scalebox{1}{$\treenum{n}$}};

\draw (x0) to[out=210, in=60] node[midway, right,yshift=-4pt, draw=none] {$\myedge{\tilde{x}_1,\tilde{y}_1}$} (x1);
\draw (x0) -- node[midway, right,xshift=-2pt, draw=none] {$\myedge{\tilde{x}_i,\tilde{y}_i}$} (xi);

\draw (x0) to[out=-30, in=120] node[midway, right,xshift=4pt, draw=none] {$\myedge{\tilde{x}_n,\tilde{y}_n}$} (xn);

\end{tikzpicture}
}
\caption{The construction of $\tree$ in the proof of \lemref{soa-mistakes}}
    \label{fig:soa-figure}

  \end{figure}

    It is easy to verify that $\tree$ is a \dfft\ of height $\Tdim(V^t, \histnum{t}) + 1$. We now prove that it is shattered by $V^t, \histnum{t}$ by verifying the properties in \defref{sdfft}. 
    To show property \ref{sdfft:label} (different labels), let $u=\node{y_u, \phi_u, x_u}$ be a node in $\tree$, and let its incoming edge be labeled $e = (x_e, y_e)$. Property \ref{sdfft:label} requires that $y_u \neq y_e$. If the source of this edge is the root node, then $(x_e,y_e) = (\tilde{x}_i,\tilde{y}_i)$ for some $i$, and $y_u = \overline{y}_i \neq \tilde{y}_i$, satisfying the property. Otherwise, $u$ is a node in one of the $\treenum{i}$, in which case the property holds because $\treenum{i}$ is shattered.
    
    To show property \ref{sdfft:edges} (edge labels), we first show that all outgoing edges satisfy one of the two possible conditions for labeled edges. 
     Consider an edge $e$ in $\tree$ labeled by $(x_e, y_e)$ with a source node $v$. If the edge is outgoing from the root, then by the construction of $\tree$, $(x_e, y_e) \in \histnum{t}$. 
     Otherwise, $e$ is a an edge that in a sub-tree taken from $\treenum{i}$. Since $\treenum{i}$ is a shattered \dfft, it satisfies property \ref{sdfft:edges}, thus one of the two options hold:
        \begin{itemize}
        \item $(x,y) \in \histnumi{t} = \histnum{t} \cup \{x_t, \overline{y}_i \}$. If $(x,y) \in \histnum{t}$, then the property holds for $\tree$ using the first case. Otherwise, $(x,y) = (x_t, \overline{y}_i)$. In this case, since the parent-child pair $\node{\bot,\bot, x_t} \rightarrow \node{\overline{y}_i, \phi_i, x_r^i}$ are in the path from the root to $v$, the property is satisfied using the second case. 
        \item In $\treenum{i}$, these is a labeled example $(x,y)$ in the path from the root to the node that was mapped to $v$ in $\tree$. From the construction, the same holds for $v$ in $\tree$.
        \end{itemize}

    Property \ref{sdfft:edges} further requires that the outgoing edges of each non-leaf node are all the possible labeled edges according to the two conditions. For the root node, this holds since the outgoing edges of this node correspond to all the labeled examples in \hist, and there are no other labeled examples
    For a node that was mapped from some $\treenum{i}$ (including its revised root), note that the only difference in the set of possible outgoing edge labels is due to the example $(x_t, \overline{y}_i)$, which is a labeled example that appears in the path to this node in $\tree$ but possibly not in $\treenum{i}$. The property holds for this node, because it holds for the mapped node in $\treenum{i}$ with history $\histnumi{t} = \histnum{t} \cup \{ (x_t, \overline{y}_i)\}$. 

    Property \ref{sdfft:consistent} requires that each path from the root to a leaf is consistent with some teacher in $V^t_{\histnum{t}}$. Every such path in $\tree$ is of the form $\node{\bot,\bot,x_t} - (\tilde{x}_i,\tilde{y}_i) \rightarrow \node{\overline{y}_i, \phi_i, x_r^i} \rightarrow p$, where $p$ is a path in $\treenum{i}$ from a child of the root to a leaf. Since $\treenum{i}$ is shattered, there is a teacher from $V^t_i \subseteq V^t_{\histnumi{t}}$, where $\histnumi{t} \supseteq \histnum{t}$, that is consistent with the path $\node{\bot,\bot,x_r^i} \rightarrow p$. From the definition of $\histnumi{t}$, it follows that the same teacher is also consistent with the path in $\tree$.
    Lastly, for Property \ref{sdfft:complete}, it is easy to see that since $\treenum{i}$ are all complete trees, then so is $\tree$. 

    We have established that $\tree$ is a shattered \dfft\ for $V^t$ and $\histnum{t}$. It follows that $\Tdim(V^t, \histnum{t}) = 1+ \Tdim(V_i^t,\histnumi{t})$ for all $i$.
    This implies that $\Tdim(V^t, \histnum{t}) = \Tdim(V^{t+1}, \histnum{t+1})+1$, in contradiction to \lemref{tdim-hist-and-teachers}. Thus, the assumption is false and \eqref{strict} holds.
\end{proof}

Lastly, we prove the second part of the theorem, \lemref{at-least}, which states that any deterministic DFF algorithms has a mistake lower bound of at least \Tdim.
\begin{proof}[Proof of \lemref{at-least}]
  Let $\cA$ a deterministic DFF algorithm. Let $\tree$ be a shattered \dfft\ for $\cT$ with history $\hist$, with height $d = \Tdim(\cT, \hist)$. 
  We show that there exists a sequence of examples and a teacher for which $\cA$ makes at least $d$ mistakes, by traversing a path from the root of $\tree$ to a leaf and considering the interaction that it describes. 

 Denote the example at the root of the tree by $x_1$. Having observed $x_1$, the algorithm  deterministically provides an explanation example $\hat{x}_1$ and a predicted label $\hat{y}_1$, where $(\hx_1, \hy_1) \in \hist$. Since $\tree$ is  shattered by $\cT,\hist$, then according to property \ref{sdfft:edges} in \defref{sdfft}, the root has an outgoing edge labeled $(\hat{x}_1,\hat{y}_1)$.
 Let  $v_2  = \node{y, \phi, x_2}$ be the target node of this edge. Set the teacher feedback for the algorihm's response to the label $y$ with feature feedback $\phi$. By property \ref{sdfft:label}, $y \neq \hat{y}_1$. Thus, $\cA$ makes a mistake in this round.

  The next example in the sequence is set to $x_2$, which appears in $v_2$. This round and subsequent rounds proceed similarly: by the definition of a shattered \dfft, any pair of explanation and label provided by the algorithm is consistent with some outgoing edge of $v_2$. The teacher feedback and the next example are set by the target node $v_3$ of this edge, causing the algorithm to make another mistake. The run ends when reaching a leaf node. Since the height of $\tree$ is $d$ and $\tree$ is complete (by property \ref{sdfft:complete}), this results in a sequence of $d$ examples, on which $\cA$ makes $d$ mistakes. By property \ref{sdfft:consistent}, the feedback provided in this run is consistent with some teacher in $\cT$. This proves the claim. 
      \end{proof}

\thmref{tdim} follows directly from \lemref{soa-mistakes} and \lemref{at-least}.

  \section{Example: The \Tdim\ of a relaxed component model}\label{app:compex}

  The following DFF learning setting is a relaxed version of the component model first studied in \cite{DasguptaDeRoSa18}. In our setting, each component is defined via a conjunction of features that hold for all examples in the component, as well as a common label. In addition, the label $0$ is reserved for examples not in any component.

\begin{example}[Relaxed component model]\label{exm:AND-RM} 
  Let $L, R, M \in \nats$. Let $\cX$ be an arbitrary set of examples, and let $\Phi \subseteq \{0,1\}^\cX$ be an arbitrary set of features that is closed under negation. For a set $S \subseteq \Phi$, we say that $x$ satisfies $S$ if $\prod_{\phi \in S} \phi(x) = 1$ and denote $S(x) := \prod_{\phi \in S} \phi(x)$. 
    Define $\cY = \{0,\ldots, L\}$. 
    Each labeling function in the teacher class is defined by a collection of (at most) $R$ sets of features $S_1,\ldots,S_R \subseteq \Phi$, where each set is of size (at most) $M$. Given $\cS = \{S_1,\ldots,S_R\}$ and a mapping $q:[R] \rightarrow [L]$,
    the label of an example $x$ is $q(j)$ if $S_j(x)$ holds. If more than one label satisfies this condition for any $x \in \cX$, then the labeling function for $(\cS,q)$ is undefined. If no $j$ satisfies $S_j(x)$ then the label of $x$ is set to $0$. Denote the labeling function for $\cS,q$ by $\ell_{\cS,q}$. 

    Let $\cL := \{ (\cS, q) \mid \ell_{\cS,q}\text{ is defined}\}$. 

For any $(\cS,q) \in \cL$, we define the set of possible teachers $\cT_{\cS,q}$, by fixing their labeling function to $\ell_{\cS,q}$, but allowing different feature feedback functions $\psi$. The feature feedback function of each teacher in $\cT_{\cS,q}$ satisfies the following:
There exists a mapping $S:\cX \rightarrow \cS$ which selects the ``primary'' conjunction for each $x$ with a non-zero label. Formally, for each $x$ such that $\ell_{\cS,q}(x) = i \neq 0$, there is some $j$ such that $S(x) = S_j$, $S_j(x)$ holds and $q(j) = i$. Denote $F(x) = \{ \phi \mid \phi(x) \text{ holds.}\}$
For a set $A \subseteq \Phi$, denote $\neg A = \{ \neg \phi \mid \phi \in A\}$. 
Then 
  the $\psi$ satisfies:
\[
  \text{If }\ell_{\cS,q}(\hat{x}) \neq 0, \psi(x,\hat{x}) \in  F(x) \cap \neg S(\hat{x}). \text{ Otherwise, }\psi(x,\hat{x}) \in  S(x) \cap \neg F(\hat{x}).
\]
Note that this implies also that $\psi(x,\hat{x})$ satisfies $x$ and does not satisfy $\hat{x}$, as required.

$\cT_{\cS, q}$ is the set of all teachers with labeling function $\ell_{\cS,q}$ and a $\psi$ function that satisfies the requirement above.
The teacher class is set to $\cT = \cup_{(\cS, q) \in \cL} \cT_{\cS,q}$.

\end{example}

  \begin{figure}[ht]
    \centering

    \resizebox{0.9\textwidth}{!}{
\begin{tikzpicture}[
  node distance=1.2cm and 2cm,
  every node/.style={draw, rounded corners, align=center},
  ->, >=Stealth
]

\node (x1) {$\node{\bot, \bot, \overline{01}}$};

\node (x2) [below=of x1] {$\node{1, f_2, 01\overline{0^31}}$};

\node (x3_0) [below left=of x2] {$\node{1,f_6, 010^31\overline{0^71}}$};
\node (x3_1) [below right=of x2] {$\node{0, \neg f_4, 0101\overline{0^71}}$};

\node (leaf1_0) [below left=of x3_0, xshift=1.5cm] {$\node{1, f_{14}, \bot}$};
\node (leaf2_0) [below=of x3_0] {$\node{0, \neg f_{10}, \bot}$};
\node (leaf3_0) [below right=of x3_0, xshift=-1.5cm] {$\node{0, \neg f_{10}, \bot}$};

\node (leaf1_1) [below left=of x3_1, xshift=1.5cm] {$\node{1, f_{12}, \bot}$};
\node (leaf2_1) [below=of x3_1] {$\node{0, \neg f_{8}, \bot}$};
\node (leaf3_1) [below right=of x3_1, xshift=-1.5cm] {$\node{1, f_{12}, \bot}$};

\draw (x1) -- node[midway, right, draw=none] {$\myedge{\overline{0},0}$} (x2);

\draw (x2) -- node[midway, right, xshift=15pt, draw=none] {$\myedge{\overline{0},0}$} (x3_0);
\draw (x2) -- node[midway, right, xshift=10pt, draw=none] {$\myedge{\overline{01},1}$} (x3_1);

\draw (x3_0) to[out=210, in=60] node[midway, right,xshift=2.5pt, yshift = -5pt, draw=none] {$\myedge{\overline{0},0}$} (leaf1_0);
\draw (x3_0) -- node[midway, right,xshift=-2pt, draw=none] {$\myedge{\overline{01},1}$} (leaf2_0);

\draw (x3_0) to[out=-30, in=120] node[midway, right,xshift=2.5pt, yshift = 4pt, draw=none] {$\myedge{01\overline{0^31}, 1}$} (leaf3_0);

\draw (x3_1) to[out=210, in=60] node[midway, right,xshift=2.5pt, yshift = -4pt, draw=none]{$\myedge{\overline{0},0}$} (leaf1_1);
\draw (x3_1) -- node[midway, right,xshift=-2pt, draw=none] {$\myedge{\overline{01},1}$} (leaf2_1);
\draw (x3_1)  to[out=-30, in=120] node[midway, right,xshift=7pt, yshift = 2pt, draw=none] {$\myedge{01\overline{0^31}, 0}$} (leaf3_1);

\end{tikzpicture}
}
\caption{An example of a shattered \dfft\ of height $3$ for the teacher
class of the relaxed component model (\exmref{AND-RM}) with $L=1$, $R=1$, $M=3$, and
$\hist=\{(\overline{0},0)\}$.}
\label{fig:relaxed-component-shattered-dfft}
\end{figure}

The following result provides an upper bound on the \Tdim\ of any example of the form in \exmref{AND-RM}.
\begin{theorem}\label{thm:AND-RM-below}
  Given $L, R, M \in \nats$, $\cX$ and $\Phi \subseteq \truthvalues^\cX$,
 let $\cT$ be defined as in \exmref{AND-RM}. Let $\hist = \{(x_0, 0)\}$ for some $x_0 \in \cX$ such that $\cT$ is consistent with $\hist$. Then $\Tdim(\cT, \hist) \leq RM$.
\end{theorem}

To prove this upper bound, we use a slight variant of the algorithm provided in \cite{DasguptaDeRoSa18} for the component model. The algorithm provided there assumes an initial arbitrary labeled example. The algorithm maintains a list of conjunctions attached to previously observed labeled examples. When a new example arrives, the algorithm finds a conjunction satisfied by the example and outputs the labeled example associated with this conjunction. If the predicted label is incorrect, the algorithm adds the negation of the resulting feature feedback to the conjunction of the labeled example it used for prediction. If no existing conjunction matches the presented example, the algorithm predicts using the default labeled example. If this prediction is incorrect, the algorithm creates a new conjunction that includes only the provided feature feedback and attaches to it the current example with its true label. The mistake bound of \citet[Lemma 4]{DasguptaDeRoSa18} for this algorithm is proved for the (original) component model, and provides a bound of $R'M'$, where $R'$ is the number of conjunctions created by the algorithm and $M'$ is the maximal length of a conjunction created by the algorithm. The proof for the relaxed version in \exmref{AND-RM} is almost identical. The differences are minor, and are delineated in the proof sketch below.

\begin{proof}[Proof Sketch (\thmref{AND-RM-below})]
  The upper bound is obtained by showing the existence of a deterministic algorithm that obtains a mistake upper bound of $RM$ for the given type of DFF problem. This is achieved by observing that the analysis that derives the mistake bound of the algorithm of \cite{DasguptaDeRoSa18} for the component model can be easily adapted for examples of type \exmref{AND-RM} as follows. The default labeled example assumed in \cite{DasguptaDeRoSa18} is replaced by the single labeled example $(x_0,0) \in \hist$. Since a conjunction is never created for the default label by the algorithm, this guarantees that all conjunctions created by the algorithm are assigned an example that satisfies one of the conjunctions $S_j$ and has a non-zero label. It follows from the definition of $\psi$ for teachers in $\cT$ that all features added to conjunctions by the algorithm do belong to the component $S(x)$ of the example $x$ attached to the conjunction: if the conjunction is pre-existing, $x= \hat{x}$ and the feature is from $S(\hat{x})$, and if the conjunction is new then $x = x_t$ and the feature is from $S(x_t)$. We conclude that the maximal length $M'$ of a conjunction created by the algorithm is $M$. In addition, the analysis of \citet[Lemma 3]{DasguptaDeRoSa18} holds, showing that the number of rules $R'$ is at most $R$. Thus, \citet[Lemma 4]{DasguptaDeRoSa18} holds for \exmref{AND-RM} and provides a mistake bound of $RM$. 
\end{proof}

Using the definition of \Tdim, we can further show that for a given choice of $L, \cX, \phi$ and $\hist$, the \Tdim\ for \exmref{AND-RM} is in fact equal to the worst-case $RM$, by constructing an appropriate shattered \dfft. This can be compared to the lower bound provided in \citet{DasguptaSa20}, which is specific to the original component model and provides a lower bound equivalent to $RM/16$. That proof used a probabilistic analysis, which we can now avoid thanks to our new \Tdim\ construction.

\begin{theorem}\label{thm:AND-RM-Tdim}
  Set $L = 1$, and set $\cX, \Phi$ to be the natural feature construction defined in \secref{online}. For $\phi = f_n$ or $\phi = \neg f_n$, call $n$ the coordinate of $\phi$. Call $f_n$ a positive feature and $\neg f_n$ a negative feature.
  Let $\hist = \{(\bar{0},0)\}$. 
  Given $R, M \in \nats$, define $S_1,\ldots,S_R \subseteq \Phi$ to be mutually exclusive sets of positive features $f_n$ of size $R$. let $\cT$ be as in \exmref{AND-RM}. Then $\Tdim(\cT,\hist) = RM$. 
\end{theorem}

\begin{proof}
  The upper bound $\Tdim(\cT,\hist) \leq RM$ follows from \thmref{AND-RM-below}. For the lower bound, we construct a shattered \dfft\ of height $RM$ for $\cT, \hist$. Set the root node of the \dfft\ to $\node{\bot, \bot, \overline{01}}$. This node has a single child, with an edge labeled $(\bar{0},0)$.

We construct the rest of the tree inductively, such that in each path, each feature feedback is different from the previous ones in the path, and also consistent with all previous example labels. 
The sequence of features provided in each path from the root to a leaf determines $\cS$ for a teacher that is consistent with this path.

Let $A_v$ be the set of all $i \in \nats$ such that for all labeled examples $(x,y)$ in the path from the root to $v$, $x(i) = y$.
Formally, the inductive construction maintains the following properties for every node $v = \node{\cdot, \cdot, x}$. 
        \begin{itemize}
            \item If $x \neq \bot$, then for any $y \in \{0, 1\}$, $|A_v \cap \{ i \in \nats \mid x(i) = y\}| = \infty$.
            \item
              Let $v_1,\ldots,v_t = v$ be the nodes in the path from the root to $v$, and denote $v_i = \node{y_{i-1}, \phi_{i-1}, x_i}$. Let $I_v = (i_1,\ldots,i_{t-1})$ be the coordinates of the features $\phi_1,\ldots,\phi_{t-1}$. 
              For any $i \in [t-1]$, let $r_i$ be the maximal integer which is smaller than $i$ and divisible by $M$. Then the following holds: $\forall j \leq r_i, x_i(i_j) = 0$, and $\forall j \in \{r_i+1,\ldots, i-1\}, x_i(i_j) = 1$. In addition, for all $j \in \{i,\ldots,t-1\}$, $x_i(i_j) = y_i$. 
For instance, if $M = 3$ and $t = 10$ then $x_9(i_1,\ldots,i_{t-1}) =(0,0,0,0,0,0,1, 1, y_i,y_i)$.               
            \end{itemize}
            These properties hold trivially for the root node. We inductively construct descendant nodes as follows. Let $v = \node{\cdot, \cdot, x}$ be a node with outgoing edges and consider an outgoing edge from $v$ labeled $(\hat{x}, \hat{y})$ where $\hat{y} \in \{0,1\}$. Assume that the inductive hypothesis holds for $v$. We set the target node of this edge to $u = \node{y, \phi, x'}$, where $y = 1-\hat{y}$  and $\phi \in \Phi, x' \in \cX$ are defined below.

            Let $\bar{A}_u := A_u \setminus I_v$. These are feature coordinates that do not appear in the path from the root to $v$, whose value in all labeled examples on the path to $u$ agrees with their labels. Note that $\bar{A}_u$ is infinite, since $A_u = A_v \cap \{ i \in \nats \mid x(i) = y\}$ is infinite by the induction hypothesis, and $I_v$ is finite.  Let $i = \min \bar{A}_u$. Then for any labeled example $(x,y)$ in the path to $u$, $x(i) = y$.
            We set $\phi$ to $f_i$ if $y=1$ and to $\neg f_i$ if $y=0$.  
            Note that since $f_i(x) = x(i) = y$, in both cases $\phi(x) = 1$. In addition, since $(\hat{x},\hat{y})$ is $(\bar{0},0)$ or a labeled example from the path to $v$, we have $f_i(\hat{x}) = \hat{x}(i) = \hat{y} = 1-y$. Hence $\phi(\hat{x}) = 0$. 

            $x' \in \cX$ is defined as follows. Let $t$ be the depth of $u$. Let $I_u = (i_1,\ldots,i_{t})$. If $t = RM$, we set $x' = \bot$ and $u$ is set to be a leaf node. Otherwise, set $r \leq t$ to be the maximal integer that is divisible by $M$. Set $x'(i_j) = 0$  for all $j \leq r$, and $x'(i_j) = 1$ for all $j \in \{r+1,\ldots,t\}$. For coordinates in $\bar{A}_u\setminus I_u$, set their values to $0$ and $1$ in alternating order, so that infinitely many of such coordinates have a value of $0$ and infinitely many of them have a value of $1$. Other coordinates of $x'$ can have arbitrary values.  Define an outgoing edge from $u$ for every pair $(\hat{x}, \hat{y})$ in the labeled examples to $u$ or in $\hist$. An illustration of some of the construction of $x'$ is provided in \figref{path}.

\begin{figure}
         \begin{tabular}{llll}
$(\bot, \bot, x_1=\overline{01}) $
&$\xrightarrow{(\overline{0}, 0)}$ & $(1, f_2,$ & $x_2=\star 1\,\overline{\star \,0\star 1})$ \\
&$\xrightarrow{(x_1, 1)}$ & $ (0, \neg f_4,$ & $x_3=\star 1\star 1\,\overline{\star^30\star^31})$\\
&$\xrightarrow{(x_2, 0)}$ & $(1, f_{12}, $ & $x_4=\star 0 \star 0 \star ^70\,\overline{\star ^70\star ^71})$ \\
&$\xrightarrow{(x_3, 1)}$ & $(0, \neg f_{28},$ & $x_5=\star 0 \star 0\star ^70\star ^{15}1\,\overline{\star ^{15} 0\star ^{15}1}) \ldots $ 
         \end{tabular}
         \caption{An example of a path prefix starting from the root in the inductive construction of the tree in the proof of \thmref{AND-RM-Tdim}, for $M=3$. The symbol `$\star$' indicates an arbitrary Boolean value. An overline indicates infinite repetition of the sequence under it.}
         \label{fig:path}
\end{figure}

It can now be verified that the induction hypothesis holds for this construction: $A_u \cap \{ i \in \nats \mid x'(i) = y\}$ is infinite from the choice of $x'$. The second property holds for $x'$ from the choice of $x'$, and for other examples on the path from the inductive hypothesis and the choice of $\phi$ as agreeing with all previous labeled examples.

            To prove that the result is a shattered \dfft\ for $\cT$ and $\hist$, it suffices to show that for any path from the root to a leaf, there is a teacher in $\cT$ that is consistent with this path. Let $v$ be the leaf. Partition $I_v$ into $R$ sub-sequences of size $M$, and define $\cS = \{ S_1,\ldots,S_R\}$, where $S_i$ includes the positive features of the coordinates in the $i$'th set of the partition. Set $q$ to the constant $1$ function. Then from the second property of the induction hypothesis, it is easy to verify that $\ell_{\cS,q}(x) = y$ for all labeled examples $(x,y)$ in the path. Moreover, the feature feedback provided along the path is consistent with some teacher in $\cT_{\cS,q}$, since it satisfies the properties required from $\psi$ in \exmref{AND-RM}. 
We have thus constructed a shattered \dfft\ of height $RM$ for $\cT,\hist$. 
\end{proof}

\section{Proofs of  Theorems \ref{thm:otd-dto}, \ref{thm:ldimequaltdim}: dimensions in online conversion}\label{app:online}

\begin{proof}[of \thmref{otd-dto}]
    Denote $(\cT_\cF, \hist) = \otdmap(\cF)$. Then $\cX \cap \hist_\first = \emptyset$, and $\cT_\cF$ is a teacher class over $\cX \cup \hist_\first, \cY$ and some $\Phi$. Denote $\cF' = \dtomap(\cT_\cF, \hist)$.
    Then $\cF'$ is a hypothesis class over $(\cX \cup \hist_\first) \setminus \hist_\first = \cX$. Therefore, $\cF$ and $\cF'$ are both subsets of $\cY^\cX$.
    
    To show that $\cF = \cF'$, let $f \in \cF$. By the definition of $\otdmap$, there is a teacher $T = (\ell, \psi) \in \cT_\cF$ such that $\ell|_{\cX} = f$. 
    By the definition of $\dtomap$, it follows that there is some $f' \in \cF'$ such that $f' = \ell|_{\cX} = f$.  It follows that $\cF \subseteq \cF'$.
    Conversely, for any $f' \in \cF'$, there is some $T = (\ell, \psi) \in \cT_\cF$ such that $\ell|_{\cX} = f'$. It follows from the definition of $\otdmap$ that there exists some $f \in \cF$ such that $\ell|_\cX = f$. Therefore, $\cF' \subseteq \cF$. It follows that $\cF = \cF'$.
\end{proof}

To prove \thmref{ldimequaltdim}, we first provide two lemmas. 
The first lemma shows that the mistake bound cannot decrease by the conversion. The second lemma provides a lower bound for a mapping in the converse direction.

\begin{lemma}\label{lem:ldim-to-tdim}
For any hypothesis class $\cF \in \cY^\cX$, $\Tdim(\otdmap(\cF)) \geq \Ldim(\cF)$.
\end{lemma}

\begin{proof} 
 Denote $(\cT_\cF, \hist) := \otdmap(\cF)$. To prove the claim, we construct a \dfft\ for $\cT_\cF$ and $\hist$ of height at least $\Ldim(\cF)$.

 Let $\tree$ be a shattered Littlestone tree for $\cF$ with height $\Ldim(\cF)$.  This is a full binary tree with internal nodes labeled by examples and 
 We construct $\tree'$, a shattered \dfft\ of the same height, by inductively mapping from nodes in $\tree$ to nodes in $\tree'$. For a node $v$ in $\tree'$, let $N(v)$, which will be defined inductively, be the node in $\tree$ that it was mapped from.

 Assume that $\tree$ is of height at least $1$, otherwise the statement holds trivially. The root node $v'_0$ of $\tree'$ is set to $\node{\bot, \bot, x_1}$, where $x_1$ is the example in the root node $v_0$ of $\tree$. We set $N(v'_0) := v_0$.
 Inductively from the root, for every node $v = \node{\cdot, \cdot, x}$ in $\tree'$ such that $x \neq \bot$, define outgoing edges labeled by $(\hat{x},\hat{y}) \in \cX \times \cY$, for each possible pair as defined in Property \ref{sdfft:edges} of \defref{sdfft}. Note that by definition, there is always at least one such pair.
 Let $y_1,y_2$ be the two outgoing edges from $N(v)$ in $\tree$, and let $x_1,x_2$ be the respective examples in the target nodes $u_1,u_2$ of these edges in $\tree$ (see \figref{ldimtdim}). Note that if $u_1,u_2$ are leaves in $\tree$ then they are not labeled by an example. In this case, set $x_1,x_2$ to $\bot$. 
 In $\tree'$, set the target of each outgoing edge $(\hat{x},\hat{y})$ to a node $\node{y_1, \one[x], x_1}$ if $\hat{y} \neq y_1$, and to $\node{y_2, \one[x], x_2}$ if $\hat{y} = y_1$. For each such node $v'$, we define $N(v') = u_i$ for the appropriate $i \in \{1,2\}$. Note that more than one node in $\tree'$ can be mapped from the same node in $\tree$.

 It is easy to verify that the resulting tree $\tree'$ is a \dfft, and that it satisfies the properties required from a shattered tree. In particular, for every path in $\tree'$, there is a teacher in $\cT$ that is consistent with it: The sequence of labeled examples in each of the paths exists also as a path in $\tree$, therefore there exists a function $f \in \cF$ that is consistent with this sequence. The labeling function provided by the teacher $T_f \in \cT$ is thus consistent with the labels on the path in $\tree'$. In addition, by definition of $T_f$ and $\tree'$, the features in each node in the path are also consistent with $T_f$. Lastly, all teachers in $\cT$ are consistent with $\hist$, as required. Since $\tree'$ is a shattered \dfft\ with the same height as $\tree$, we have $\Tdim(\otdmap(\cF)) \geq \Ldim(\cF)$.
\end{proof}

Next, we show lower-bound on the Littlestone dimension for the converse conversion. 
\begin{lemma}\label{lem:tdim-to-ldim}
Given a teacher class $\cT$ and a history $H$, let $N$ be the number of labels from $\cY$ that do not appear in any pair in $\hist$. Then
    \[
    \Ldim(\dtomap(\cT, \hist)) \geq \Tdim(\cT, \hist) - N.
    \]
\end{lemma}

\begin{proof}
    Denote $\cF := \dtomap(\cT, \hist)$. To prove the claim, we construct a shattered Littlestone tree for $\cF$ of height $\Tdim(\cT, \hist) + |\cY| - \{y \mid (\cdot, y) \in \hist\}$.
   
    Let $\tree$ be a shattered \dfft\ of height $\Tdim(\cT, \hist)$.
    Assume that $\tree$ is of height at least $1$, otherwise the statement holds trivially.
    We construct a Littlestone tree $\tree'$ for $\cF$ inductively, as follows.
    For a node $v$ in $\tree'$, let $N(v)$, which will be defined inductively, be the node in $\tree$ that it was mapped from.
    
    Label the root node $u$ of $\tree'$ with the example in the root $v$ of $\tree$. We set $N(u) = v$. Inductively from the root, for a node $u$ in $\tree'$ which was mapped from some $v$, consider the children of $v$ in $\tree$.
    If there are two child nodes of $v$ with two different labels $y \neq y'$, then arbitrarily select two such nodes $\node{y, \phi, x}$ and $\node{y', \phi', x'}$ and set two outgoing edges from $u$ labeled $y$ and $y'$. Label the target nodes of these edges by $x$ and $x'$, respectively. If $x,x' = \bot$, then set the target nodes of the edges to unlabeled leaf nodes.

    If all the child nodes of $v$ are labeled with the same label, arbitrarily select one of $v$'s child nodes, and use its child nodes instead of $v$'s child nodes. This may be repeated, until a descendant of $v$ with two child nodes with different labels is found. After constructing the whole tree, prune paths from root to leaf in $\tree'$ so that they all have the length of the minimal-length path, resulting in a complete binary tree. 
    
    To verify that $\tree'$ is a Littlestone tree, note that each path in $\tree$ corresponds to the labeling function of some teacher in $\cT$, and therefore, by the definition of $\dtomap$, also to some function in $\cF$. Moreover, the height of $\tree'$ is at least $\Tdim(\cT, \hist) - N$. This is because the only time that a path in $\tree$ is shortened by $1$ for $\tree'$ is if all labels of a node $v$'s child nodes are the same label $y$. However, by the definition of a shattered \dfft, this can only be the case if none of the outgoing edges from $v$ are labeled with $y$. This is only possible if $y$ is not a label in $\hist$ and also it is not in any labeled example in the path to $v$. However, after one occurrence of the latter, $y$ is now in a labeled example in the path. Therefore, each label not in $\hist$ can cause at most a decrease of $1$ in the length of the path. This completes the proof.
      \end{proof}

  Using the lemmas above, \thmref{ldimequaltdim} can be easily proved.

  \begin{proof}[Proof of \thmref{ldimequaltdim}] First, by \lemref{ldim-to-tdim}, $\Ldim(\cF) \leq \Tdim(\otdmap(\cF))$. For the other direction of the inequality, Let $(\cT, \hist) = \otdmap(\cF)$. By the definition of $\otdmap$, all the labels in $\cY$ appear in $\hist$. Therefore, by \lemref{tdim-to-ldim},
  $\Ldim(\dtomap(\cT, \hist)) \geq \Tdim(\cT, \hist)$. Equivalently,
  $\Ldim(\dtomap(\otdmap(\cF)) \geq \Tdim(\otdmap(\cF))$. By \thmref{otd-dto}, $\dtomap(\otdmap(\cF)) = \cF$, hence $\Ldim(\cF) \geq \Tdim(\otdmap(\cF))$. 
\end{proof}

\section{A non-realizable mistake upper bound for general interactive protocols}\label{app:interactive}

We state and prove \thmref{gaa}, stated informally above as \thmref{informal}. \thmref{gaa} is a generalization of the $k$-non-realizable upper bound of \thmref{upper} to a wide range of interactive protocols. We consider a general interactive prediction protocol in which each round $t$ is of the following form: 
\begin{enumerate}
  \item The environment provides some $x_t$;
  \item The algorithm outputs $z_t$;
    \item The environment provides feedback $w_t$.
    \end{enumerate}
    We further assume some initial input $I$ that is provided before the start of the first round, and a class of generalized teachers $\cT$ (not to be confused with the specific notion of teachers in DFF), whose role is defined below. The protocol defines the following rules of interaction between the teacher and the algorithm.
    
    \begin{itemize}

      \item The legal responses of the algorithm, given the example presented in the current round and the run so far. Formally, given $D = \{I, x_t, ((x_1,z_1,w_1), \ldots, (x_{t-1},z_{t-1},w_{t-1}))\},$
      the protocol defines $A(D)$, the set of values that the algorithm may provide as $z_t$.

        \item For every possible teacher $T$, the protocol defines for each $x_t$ and $z_t$ which feedback values $w_t$ are consistent with $T$.
        \end{itemize}
        Many interactive protocols fall under this definition, including, for instance, Online Learning, Online Selective Classification with Limited Feedback \citep{gangrade2021online}, bandit Online Learning \citep{DanielySaBeSh15}, and learning with per-sample side information \citep{VisotskyAtCh19}. In particular, DFF is an instance of this general protocol, with $I := \hist$.               

        \thmref{gaa} is a generalized form of \thmref{upper}, which holds for interactive protocols as defined above that satisfy the following natural assumptions. These assumptions hold for DFF, as well as the other protocols mentioned above. 

              \begin{assumption}\label{mistakes}
                There exists a fixed known function that given $(x_t, z_t,w_t)$ for a given round $t$, indicates whether the algorithm made a mistake in this round. 
              \end{assumption}

              \begin{assumption}\label{include}
          The outputs $z_t$ that the algorithm is allowed to provide if executed on a suffix of another run, are a subset of those allowed in the original run. 
          Formally, let $t,l \in \nats$ such that $t \geq l$. Given $I, \{(x_i,z_i,w_i)\}_{i \in [t-1]}$ and $x_t$, let
          \begin{align*}
            D &= \{I, x_t, ((x_1,z_1,w_1), \ldots, (x_{t-1},z_{t-1},w_{t-1}))\}, \\ 
            D' &= \{I, x_t, ((x_l,z_l,w_l), \ldots, (x_{t-1},z_{t-1},w_{t-1}))\}.
          \end{align*}
          Then $A(D') \subseteq A(D)$. 
        \end{assumption}

                        In a run with $L$ rounds, the number of protocol exceptions is the size of the smallest subset  $E \subseteq [L]$ such that the feedbacks $w_t$ for $t \in [L]\setminus E$ are consistent with some $T \in \cT$.        Now, suppose that there is an algorithm $\cA$ for the given protocol and a function $\cM$ that maps teacher classes to natural numbers, such that $\cA(\cT, I)$ makes at most $\cM(\cT, I)$ incorrect predictions in a run with zero exceptions. In the case of DFF, $\cA$ can be set to \soadff\ and $\cM(\cT, I)$ to $\Tdim(\cT, I)$.
    To handle exceptions, one can use the algorithm in \gaa\ listed in \myalgref{gaa}, which is a general form of \asoadff\ presented in \secref{agnostic}.

    \begin{algorithm}
\caption{Generic Agnostic Algorithm}
\begin{algorithmic}[1]
  \Procedure{\gaa}{$\cA, \cT, I$}
  \State Initialize $\cA(\cT,I)$ and set $M \leftarrow 0$.
    \For{$t=1,2,\dots$}
    \State Receive $x_t$ and provide it to $\cA$; Record the $z_t$ provided by $\cA$.
    \State Output $z_t$
    \State Receive feedback $w_t$.
    \State If $(x_t,z_t,w_t)$ indicate a mistake, $M \leftarrow M+ 1$.
    \State If $M = \cM(\cT, \hist) + 1$, re-initialize $\cA(\cT,I)$ and set $M \leftarrow 0$
    \EndFor

\EndProcedure
\end{algorithmic}
\label{alg:gaa}
\end{algorithm}

\begin{theorem}\label{thm:gaa}
If the interactive protocol satisfies Assumptions \ref{mistakes} and \ref{include}, then in a run with at most $k$ exceptions, \gaa\ makes at most $(k + 1) \cM(\cT,I) + k$ mistakes.
\end{theorem}
\begin{proof}
  First, observe that by \ref{mistakes} $\cA$ can use $(x_t,z_t,w_t)$ to identify mistakes. Moreover, if $\cA$ is a valid algorithm for the given interactive protocol, then under Assumption \ref{include} so is $\gaa(\cA,\cT, I)$. 

Call each sub-sequence of the run between initializations of $\cA$ a ``run segment'' and denote $d = M(\cT,I)$. In each run segment, $\cA$ makes $d+1$ mistakes. Let $N$ be the total number of run segments. Then the total number of mistakes made by the algorithm is at most $N(d+1) + d$. Since the realizable mistake bound of $\cA(P)$ is $d$, all teachers in $\cT$ are inconsistent with all of the $N$ run segments. Therefore, any such teacher is inconsistent with at least $N$ rounds in the whole run. If there are at most $k$ exceptions during the run, it follows that $N \leq k$. Thus, $\gaa(\cA,\cT, I)$ makes at most $k(d+1) + d = (k+1)d+k$ mistakes.
    \end{proof}

\section{Proof of  \thmref{lowerbound}: a non-realizable lower bound}\label{app:prooflowerbound}
\newcommand{\todonotation}[1]{\todo[color=lightgray]{#1}}
\newcommand{\Hcal}{\mathcal{H}}
\newcommand{\Fcal}{\mathcal{F}}
\newcommand{\T}{\mathcal{T}}
\newcommand{\naturals}{\mathbb{N}}
\newcommand{\ignore}[1]{}

In this section, we provide the full proof of \thmref{lowerbound}.
In \secref{secret} we give background on the secret-sharing scheme that we are using. We describe the construction of the teacher classes in \secref{construction}. We prove an upper bound on the \Tdim\ of these classes in \secref{uptdim}, \lemref{kupper}. In \secref{klower}, we prove \lemref{klower}, which provides a lower bound on the mistake bound of any algorithm on these classes. \thmref{lowerbound} directly follows from \lemref{kupper} and \lemref{klower}.

\subsection{Background: The secret-sharing scheme}\label{sec:secret}

The secret sharing scheme of \cite{DBLP:journals/cacm/Shamir79} works as follows: Let $\bF_p$ denote the field of size $p$ for some prime number $p$, and consider a polynomial over $\bF_p$ of degree $d-1$ with coefficients in $\bF_p$, $P := (c_0 + \sum_{j \in [d-1]} c_jx^j)  \text{ mod } p$. Given a secret $s\in \bF_p$, arbitrary coefficients $c_1,\ldots,c_{d-1} \in \bF_p$ are selected, and we define $c_0 = s$, so that $P(0) = s$. The partial secrets are then defined as $\bar{s} = \{ s_i \}_{i \in [p]}$, where $s_i = (i,P(i))$.
    
    According to Lagrange's interpolation theorem, the polynomial $P$ can be reconstructed from any $d$ shares of the secret $s_{i_j} = (i_1,z_{i_j})$ for $j \in [d]$ via the Lagrange interpolating polynomial:
    \[ P(x) = \sum_{j=1}^{d} z_{i_j} \prod_{l \in [d]: l\neq j} \frac{  x - i_{l}}{ i_j - i_l} \textrm{ mod } p. \]
    Lagrange's interpolation theorem states that this polynomial is the unique polynomial of degree $d-1$ that satisfies $\forall j \in [d]$: $P(i_j) = z_{i_j}$.
    Once $P$ is found, one can obtain the secret $s$ by evaluating $P(0)$.
        However, for any set of $d-1$ partial secrets $s_{i'_{1}},  \dots, s_{i'_{d-1}}$, and every $s' \in \bF_p$, there exists a polynomial $P'$ with $P'(0) = s'$ such that for all $j\in [d-1]$, $P'(i'_j) = P(i'_j)$. Thus, informally, it is not possible to infer anything useful on $s$ from $d-1$ partial secrets. 
This mapping thus describes a $(d,p)$-threshold secret sharing scheme. 

\subsection{Construction of the teacher class}\label{sec:construction}

Given a natural number  $d \geq 1$, we define the teacher class $\T^{d+1}$ as follows. Let the domain $\X = \mathbb{N}$ and the label set be $\mathcal{Y} = \{0,1\}$. Set the feature set to $\Phi = \{0,1\}^\X$. Partition the set $\X$ into mutually exclusive sets $\X_i = [2^{i+1}-1] \setminus [2^i-1]$, such that $\X = \bigcup_{i=0}^{\infty} \X_i$.  We construct a teacher class that employs the secret-sharing idea described in \secref{agnostic} above. We define $\T^{d+1}$ by first defining teacher classes $\cT^d_i$ for $i \in \nats$. For a given $i$, 
we define a class of labeling functions $\Fcal^d_i$. Then, for a given $f\in \Fcal^d_i$, we define the set of all compatible feature feedback functions $\Psi_f$. We then define $\cT^d_i = \{(f,\psi) : f\in \Fcal^d_i, \psi \in \Psi_f \}$, and set $\cT^{d+1} = \cup_{i=1}^\infty \cT^d_i$. 

The class $\cF^d_i$ includes all the functions $f:\cX \rightarrow \{0,1\}$ that satisfy the following:
\begin{itemize}
\item $f(1) = 1$
  \item   For all $x \in \X \setminus(\cX_i \cup \{1\})$, $f(x) = 0$.
\end{itemize}
Thus, there is exactly one function in $\cF^d_i$ for each possible labeling of $\cX_i$.

For a given $f$, we construct $\Psi_f$, the class of possible feature-feedback functions, such that the $\Psi_f$ are mutually exclusive over all $f \in \cup_{i=0}^\infty\cF^d_i$ for any $i$. To satisfy the construction properties as described in \secref{agnostic}, Let $p_i$ be a prime number such that $p_i \geq |\Fcal^d_i| = 2^{|\cX_i|}$. We construct $\Psi_f$ to correspond to a $(d,p_i)$-threshold secret sharing scheme for $f \in \bF_i$.

We first present the general form of the functions in $\Psi_f$, using secret-sharing polynomials. For coefficients $\bar{c} = (c_0,\ldots,c_{d-1}) \in \bF^d_{p_i}$, define the secret-sharing polynomial of degree $d-1$ over $\bF_{p_i}$, $P'[\bar{c}] := (c_0 + \sum_{j \in [d-1]} c_jx^j)  \text{ mod } {p_i}$, whose secret is $P'(0) = c_0$. We then define a ``shifted'' version of this polynomial, $P[\bar{c}]:\cX_i \rightarrow \cX \setminus \cX_i$, via
\[
  P[\bar{c}](x) = 2^{i+1} + P'[\bar{c}](x-2^i+1).
  \]
For given elements $x_1,\ldots,x_l \in \cX$, denote by $\one[x_1,\ldots,x_l]:\cX \rightarrow \{0,1\}$  the feature that holds if and only if $x \in  \{x_1,\ldots,x_l\}$.
 Now, given a function $\bar{a}: \X_i \times \{0,1\} \to \bF^d_{p_i}$, define the feature feedback function
$\psi^f_{\bar{a}}: \X \times \X \rightarrow \Phi$ by

\begin{equation}\label{eq:secretpsi}
     \psi^f_{\bar{a}}(x,\hat{x}) = \begin{cases}
    \one[x,P[\bar{a}(\hat{x},f(x))](x)\,] & x,\hat{x} \in \cX_i\\
    \one[x] & \text{otherwise}.
      \end{cases}
\end{equation}
Thus, the feature feedback function defines a separate secret polynomial for every $\hat{x}$ and every label of $x$, and provides an evaluation of this polynomial on $x$. Note that the first and second element in the first case are never the same, since the shifted polynomials' outputs are not in $\cX_i$. This also implies that if $x \neq \hat{x}$ then $\psi_{\bar{a}}^{f}(x,\hat{x})$ satisfies $x$ and does not satisfy $\hat{x}$, as required.

To define the set $\Psi_f$, we encode the hypotheses in $\Fcal^d_i$ as numbers via a  fixed bijective function \mbox{$b_i: \Fcal^d_i \rightarrow \bF_{p_i}$}.
We then include in $\Psi_f$ feature feedback functions that encode $b_i(f)$ in some of its secret-sharing polynomials. The rest of the polynomials encode an arbitrary ``fake'' secret.
\[
\Psi_{f} =\{\psi_{\bar{a}}^{f} \mid \bar{a}: \X_i \times \{0,1\} \to \bF^d_{p_i}, \text{ and } \forall \hat{x} \in \cX_i, f(\hat{x}) = 1 \Rightarrow \bar{a}(\hat{x},0)(0)=b_i(f).\}.
\]
Thus, for feature feedback functions in $\Psi_f$, a part of the secret $b_i(f)$ is revealed  given $x,\hat{x}$ if and only if $f(\hat{x}) = 1$ and $f(x) = 0$. The rest of the cases reveal evaluations of a ``fake'' polynomial that does not encode the secret. We show below that this construction guarantees that the algorithm must observe $d$ true feature feedback responses for the same explanation $\hat{x}$ to reveal any information on $f$ beyond the labels of observed examples.

\subsection{Upper bounding the \Tdim\ of the construction}\label{sec:uptdim}

The following result gives an upper bound on the DFF dimension of the teacher classes defined in the construction above. Recall that $\thist = \{(2,0),(1,1)\}$, as defined in \thmref{lowerbound}.

\begin{lemma}\label{lem:kupper}
Let $i \in \nats$. Suppose that $\hist = \{(x_0^0, 0), (x_0^1, 1)\}$ for some $x_0^0, x_0^1 \in \nats$  such that $\hist$ is consistent with some teacher in $\cT_i \subseteq \cT^{d+1}$ for some $i \in \nats$. Then $\Tdim(\T^d_i, \hist) \leq d$ and $\Tdim(\T^{d+1}, \hist) \leq d+1$. 
In particular, the history $\hist_\bullet$ satisfies this for all $i \in \nats$. 

\end{lemma}

\begin{proof}
    We prove the result for $\cT^{d+1}$ by providing a deterministic algorithm for this class with a mistake bound of $d+1$ in the realizable case. By \thmref{tdim}, this implies that $\Tdim(\T^{d+1}) \leq d+1$. We then show how to adapt the algorithm to obtain a mistake bound of $d$ if the teacher class is $\cT^d_i$.

    The  algorithm for $\cT^{d+1}$, denoted $\mathcal{A}$, is provided in \myalgref{soadfflb}. The algorithm predicts $0$ for any unknown example until the first example with a label of $1$ that is not $x= 1$ is observed (note that if $x_0^1 \neq 1$, this occurs as soon as the algorithm starts). Denote the true teacher by $T^* = (f^*, \psi^{f^*}_{\bar{a}})$. Once $n$ is defined, for any example with an unknown label, the algorithm predicts $1$, using this first example as explanation. It collects the partial secret obtained from the feature feedback for each mistake, and after collecting $d$ such parts, it reconstructs the true labeling function and uses it to label all subsequent examples. 
    We write ``output (x,y)'' to indicate that the algorithm predicts label $y$ with explanation $x$.

\algnewcommand{\IfThen}[1]{\State\algorithmicif\ #1\ \algorithmicthen}
\algnewcommand{\EndIIf}{\unskip\ \algorithmicend\ \algorithmicif}
    
\begin{algorithm}[t]
\caption{A Learning algorithm for $\T^{d+1}$, $\hist = \{(x_0^0,0),(x_0^1,1)\}$.}
\begin{algorithmic}[1]
    \State Initialize $R\gets \emptyset$. If $x_0^1 \neq 1$, set $\hat{x} \gets x_0^1$ and define $n$ to the unique integer such that $\hat{x} \in \X_n$.
    \For{$t= 1,..., L$}
    \If{$\hat{f}$ is defined} output $(x_0^{\hat{f}(x_t)}, \hat{f}(x_t))$.
    \ElsIf{$x_t = 1$} output $(x_0^1, 1)$.
    \ElsIf{there is a positive integer $s< t$ such that $x_s = x_t$} output  $(x_s,y_s)$.
        \ElsIf{$|R| < d$}
                     \If{$n$ is undefined}
                                \State output $(x_0^0,0)$.
                                \If{mistake} set $\hat{x} \gets x_t$ and define $n$ to the unique integer with $x_t\in \X_n$.\EndIf\label{ichangeline}
                            \ElsIf{$x_t \in \X_{n}$}

                                            \State output  $(\hat{x}, 1)$.                 
                                            \If{mistake}
                                                \State receive feature feedback $\one[x_t, z_t]$,

                                                \State update $R \gets R \cup \{(x_t - 2^n +1, z_t - 2^{n+1})\}$.
                                             \EndIf
                                                                 
                             \Else\ ($n$ is defined, $x_t \notin \cX_n$)
                                        \State output $(x_0^0,0)$.
                                        \State If mistake, ignore feature feedback.
                                
                            \EndIf

                \Else\ ($|R|=d$)
                        \State Denote $R =\{(\tilde{x}_1, \tilde{z}_1), \dots, (\tilde{x}_d, \tilde{z}_d) \}$.
                        \State Define the polynomial $\hat{P}$ via Lagrange interpolation:
                             \[ \hat{P}(x) := \sum_{j=1}^d \tilde{z}_{j} \prod_{l \in [d]: l\neq j} \frac{x - \tilde{x}_{j}}{\tilde{x}_{l} - \tilde{x}_{j}} \text{ mod } p_{n}.\]
                        \State Define $\hat{f} \leftarrow b_{n}^{-1}(\hat{P}(0))$. 
                        \State Output $(x_0^{\hat{f}(x_t)}, \hat{f}(x_t))$.
                \EndIf        
    \EndFor
\end{algorithmic}
\label{alg:soadfflb}
\end{algorithm}

We
now prove that this algorithm makes at most $d+1$ mistakes.
First, the algorithm predicts $0$ for all previously unobserved examples (except for $x_t = 1$) until making its first mistake. Therefore, this mistake reveals an example $x_t \neq 1$ such that $f^*(x_t) = 1$.
In this round, the algorithm sets $\hat{x} \gets x_t$ and $n$ is set such that $\hat{x} \in \cX_n$ (line \ref{ichangeline}). It follows
that $f^*\in \cF_n$. Therefore, after this round, the only
unknown labels  are of examples in $\X_n$, and the algorithm makes mistakes only on such examples.

Now, suppose that $|R| < d$ and a new example $x_t$ from $\cX_n$ is observed. The algorithm outputs $(\hat{x},1)$. If this is a mistake, then $f^*(x_t)=0$ and the algorithm received a feature feedback $\phi_t = \psi^{f^*}_{\bar{a},\hat{x}}(\hat{x},x_t) = \one[x, P[\bar{a}(\hat{x},0)](x_t)]$ that provides an evaluation of $ P[\bar{a}(\hat{x},0)]$ for an $x_t$ that has never been seen before. 
At most $d$ such rounds are possible.
If $|R|=d$, then the algorithm uses the $d$ feedbacks received so far to reconstruct $f^*$. This can be done because $R$ includes $d$ distinct evaluations of the same polynomial $P[\bar{a},(\hat{x},0)]$, where $\bar{a}(\hat{x},0) = b_{n}(f^*)$, since $f^*(\hat{x})=1$. Thereafter, the algorithm makes no more mistakes, since $\hat{f} = b^{-1}(\hat{P}(0)) = f^*$ for $\hat{P} = P'[\bar{a},(\hat{x},0)]$. Taking everything together, we see that the algorithm can make at most $d+1$ mistakes.

To complete the proof, the above algorithm can be adapted for the teacher class $\T_i$ by initializing $n$ to $i$. In this case, the algorithm makes at most $d$ mistakes since it does not make the mistake in the round where $n$ is set.
\end{proof}

\subsection{Lower-bounding the number of mistakes of the construction}\label{sec:klower}

\input{klower_arxiv.tex}

\end{document}

%% file: prel.tex
\usepackage{amsfonts,stmaryrd,enumitem}
\usepackage{algorithm,booktabs,todonotes}

\usepackage{algorithmicx}
\usepackage[noend]{algpseudocode}
\usepackage{graphicx}

\usepackage{url}
\usepackage{hyperref}
\algnewcommand\algorithmicinput{\textbf{Input:}}
\algnewcommand\INPUT{\item[\algorithmicinput]}

\algnewcommand\algorithmicoutput{\textbf{Output:}}
\algnewcommand\OUTPUT{\item[\algorithmicoutput]}

\usepackage{tikz}
\usetikzlibrary{arrows}
\usetikzlibrary{calc}
\usetikzlibrary{shapes}
\usetikzlibrary{fit}
\usetikzlibrary{matrix} 
\usetikzlibrary{positioning}
\usetikzlibrary{decorations.pathreplacing}

\newtheorem{assumption}{Assumption}[section]

\renewcommand{\eqref}[1]{Eq.~(\ref{eq:#1})}
\newcommand{\figref}[1]{Figure \ref{fig:#1}}
        
\newcommand{\secref}[1]{Section \ref{sec:#1}}
\newcommand{\thmref}[1]{Theorem \ref{thm:#1}}
\newcommand{\lemref}[1]{Lemma \ref{lem:#1}}

\newcommand{\defref}[1]{Def.~\ref{def:#1}}
\newcommand{\appref}[1]{Appendix \ref{app:#1}}

\newcommand{\myalgref}[1]{Alg. \ref{alg:#1}}

\newcommand{\exmref}[1]{Example \ref{exm:#1}}
\input{prel_def.tex}

\input{prel_cal.tex}

\newcommand{\Tdim}{\ensuremath{\mathrm{DFFdim}}}
\newcommand{\Ldim}{\ensuremath{\mathrm{Ldim}}}

\newcommand{\dfft}{DFFT}

\newcommand{\node}[1]{\langle #1 \rangle}
\newcommand{\hist}{H}
\newcommand{\histnum}[1]{{\hist^{#1}}}
\newcommand{\histnumi}[1]{\hist^{#1}_i}
\newcommand{\tree}{\texttt{Tr}}
\newcommand{\treenum}[1]{{\tree_{#1}}}
\newcommand{\otdmap}{\ensuremath{\mathrm{OtD}}}
\newcommand{\dtomap}{\ensuremath{\mathrm{DtO}}}

%% file: prel_def.tex
\newcommand{\nats}{\mathbb{N}}

\newcommand{\one}{\mathbb{I}}

\DeclareMathOperator*{\argmin}{argmin}

\newcommand{\st}{\text{ s.t. }}

\renewcommand{\vec}[1]{\mathbf{#1}}

\newcommand{\false}{\texttt{false}}
\newcommand{\true}{\texttt{true}}

%% file: prel_cal.tex
\newcommand{\cA}{\mathcal{A}}

\newcommand{\cF}{\mathcal{F}}

\newcommand{\cH}{\mathcal{H}}

\newcommand{\cL}{\mathcal{L}}
\newcommand{\cM}{\mathcal{M}}

\newcommand{\cS}{\mathcal{S}}
\newcommand{\cT}{\mathcal{T}}

\newcommand{\cV}{\mathcal{V}}

\newcommand{\cX}{\mathcal{X}}
\newcommand{\cY}{\mathcal{Y}}

%% file: klower_arxiv.tex
In this section, we prove a worst-case lower bound for the number of mistakes any algorithm would make when learning the class $\T^{d+1}$ defined in \secref{construction}, with the history $H_{\bullet}= \{(2,0),(1,1)\}$ defined in \thmref{lowerbound}. This lower bound holds for adaptive adversaries. To make this precise, we introduce our assumptions on the interaction between the algorithm and the adversary below. We denote the sequence of interactions over the first $t$ rounds of a run by $S_t=((x_1,y_1,\hat{x}_1, \hat{y}_1, \phi_1), \dots, (x_{t-1},y_{t-1},\hat{x}_{t}, \hat{y}_{t}, \phi_{t}))$, where we denote $\phi_i = \bot$ whenever $\hat{y}_i = y_i$. We assume the following order of decisions by the algorithm and the adversary: In each round $t$, the adversary first selects $(x_t,y_t, \psi_t)$, where $\psi_t:\cX \times \cX \rightarrow \Phi$. The choice of this triplet can be dependent on $S_{t-1}$. Then the algorithm selects $(\hat{x}_t, \hat{y}_t)$, where its decision can depend on $S_{t-1},x_t$. Lastly, $\phi_t = \psi_t(x_t, \hat{x}_t)$ is presented to the algorithm. 
The algorithm can be randomized, thus $(\hat{x}_t, \hat{y}_t)$ can be drawn from a distribution that depends on $S_{t-1},x_t$.

Let $L$ be an integer. For a sequence $S := S_L$, $S$ is said to be $k$-consistent with a teacher $T=(f,\psi)$ if there exists a set $E\subseteq [L]$ such that $|E|\leq k$ and for all $t\in [L] \setminus E$, we have $f(x_t)=y_t$, and if $\hat{y}_t \neq y_t$ then $\psi(x_t,\hat{x}_t) = \phi_t$. 
For a given algorithm $\cA$ and an adversary $V$, let $M^L(\cA ,V, \hist)$ be the expected number of prediction mistakes made by $\cA$ when interacting with $V$ over a sequence of length $L$.  The expected worst-case number of mistakes of a randomized algorithm $\mathcal{A}$ in the $k$-non-realizable setting for a teacher class $\cT$ is defined as
   \[
     M^L_{k}(\mathcal{A} ,\T, H) = \sup_{V \in \cV(\cT,\cA)} M^L(\mathcal{A} ,V,H),
   \]
   where $\cV(\cT,\cA)$ is the set of adversaries such that when they interact with $\cA$, with probability $1$ the generated sequence $S$ is $k$-consistent with some teacher $T \in \cT$. 

   To compare this model to the randomized model of Online Learning, note that
   DFF with uninformative explanations is equivalent to Online Learning, as
   shown in \secref{online}. When applying the DFF adversary defined above to
   Online Learning via this mapping, this recovers Online Learning with an
   \emph{adaptive} adversary (as defined in \citealp{AroraDeTe12}) who makes
   up to $k$ labeling errors. An adaptive adversary for Online Learning obtains the same optimal mistake bound of $k + \Theta\left(\sqrt{k\cdot \Ldim(\cF)} + \Ldim(\cF)\right)$ as an adversary who fixes the sequence in advance, using the Weighted Majority algorithm \citep{LittlestoneWa94}. Thus, our lower bound should be compared to this lower and upper bound for Online Learning.

We prove the following lower bound for the $k$-non-realizable setting.
 
 \begin{lemma}\label{lem:klower}
     Let $k,d \geq 1$ and let $L \geq 4(k+1)d$. Let $\cA$ be a (possibly randomized) algorithm. Then 
      \[M^L_k(\mathcal{A}, \T^{d+1}, H_{\bullet} ) \geq (k+1)d -1.\]
 \end{lemma}

We prove this result by presenting a strategy for the adversary $V$ given a learning algorithm $\mathcal{A}$, as well as $L$ and $k$. We then show that this adversary only generates sequences that are $k$-consistent with $\T^{d+1}$, and prove a lower bound for $M^L_k(\mathcal{A},V,H_{\bullet})$.

Given a sequence $S$, define the version space
    $\T_S = \{ T \in \T \mid  S \text{ is (fully) consistent with }T\}.$
For a given set of exception rounds $E\subseteq [L]$, denote $S^{\setminus E} = ((x_t,y_t, \hat{x}_t, \hat{y}_t, \phi_t))_{t \in [L]\setminus E}$.
We further denote $S_{\X} := \{x_i\}_{i \in [L]}.$
Lastly, for a sequence $S$ we denote the set of rounds in which the same explanation $\hat{x}$ is given for examples with the same true label  $y$, and in which a feature feedback is provided, by

\[
  R(S,\hat{x},y) := \{ t \in [L] \mid \hat{x}_t = \hat{x}, y_t = y \text{ and } \phi_t \neq \perp\}.
  \]

  \paragraph{The adversary $V$.} Given $\cA$, $k$ and $L \geq 2$, the adversary $V$ is defined as follows. Fix $i := L$. $V$ only presents examples $x_t$ from $\X_i$, and the generated sequences will all be $k$-consistent with teachers in $\cT := \cT_i^d$. More precisely, $V$ will be $k$-consistent with a teacher for which the exception rounds are $E_L\subseteq [L]$, where $E_t$ for $t\in [L]$ is defined inductively throughout the run. Set $E_0 := \emptyset$. $E_t$ for $t \geq 1$ is defined inductively below. In round $t$, the adversary makes the following choices:

  \begin{itemize}
    \item  $x_t := 2^i + t$.

      \item We now define the choice of $y_t$ and $\psi_t$. For $y \in \{0,1\}$, denote the teachers consistent with rounds not declared as exceptions so far, as well as with $y$ as a label for $x_t$, by 
  \[
    \cT_y = \{ T=(f,\psi) \in \cT \mid T \in \cT_{S_{t-1}^{\setminus E_{t-1}}} \text{ and } f(x_t) = y\}.
  \]
  It will be apparent from the construction that $\cT_{S_{t-1}^{\setminus E_{t-1}}} = \cT_0 \cup \cT_1$ is non-empty for every $t$.
  \begin{itemize}
  \item If only one of $\cT_0,\cT_1$ is non-empty, the adversary selects $y_t$ such that $\cT_{y_t}$ is non-empty.
  \item Otherwise, it selects $y_t = y$ such that the probability that $\cA$ outputs $\hat{y}_t = y$ given $S_{t-1}$ and $x_t$ is smallest.
  \end{itemize}
  The adversary further selects some $\psi_t$ such that $(f,\psi_t) \in \cT_{y_t}$ for some $f$. 

  \item After observing $(\hat{y}_t,\hat{x}_t)$, $E_t$ is set as follows: If $|E_{t-1}| < k$ and $\hat{x}_t \in \X_i$ and $|R(S_t,\hat{x}_t,0) | = d$ then  $E_t := E_{t-1} \cup \{s\}$, where $s < t$ is the round such that $\hat{x}_t = x_s$ (note that since $i = L \geq 2$, $\hat{x}_t \in \cX_i$ cannot have come from $\thist$).  Otherwise, $E_t := E_{t-1}$. 
  \end{itemize}
  Note that since $E_t$ never decreases, and $y_t,\psi_t$ in round $t$ are selected to be consistent with some teacher in $\cT_{S_{t-1}^{\setminus E_{t-1}}}$, it follows that $\cT_{S_{t}^{\setminus E_{t}}}$ is also non-empty. In addition, since $|E_t|$ never exceeds $k$ and $\cT_{S_{L}^{\setminus E_{L}}}$ is non-empty, it follows that only $k$-consistent sequences are generated using this adversary.

  To prove a lower bound on the number of mistakes that $\cA$ makes when interacting with $V$, we first prove that $\cA$ can only have repeated a non-exception explanation for a true label of $0$ for $d$ or more times after it has made at least $(k+1)d$ mistakes.

\begin{lemma}\label{lem:kd}
  Fix $k$, let $L\geq 2$ and $t\in [L]$. Let $S =S_t$ be a sequence generated by an interaction between a learner $\mathcal{A}$ and the adversary $V$ over $t$ rounds, with history $H_{\bullet}$ and exception rounds $E=E_t$.
    If there is an element $x_s \in S^{\setminus E}_{\X} $ such that $|R(S,x_s,0)| \geq d$, then the number of prediction mistakes in $S$ is at least $(k+1)d$. 
\end{lemma}
\begin{proof}
  Assume that the condition holds. Note that $x_s \in \cX_i$. Since $x_s \in S^{\setminus E}_\cX$, it follows that $s \notin E$. Let $r$ be the first round in which $|R(S_r,x_s,0)| = d$. Then in this round, $s$ was not added to $E_r$. By the definition of the adversary, this can only happen if $|E_{r-1}| \geq k$.  Thus, $|E| \geq k$.
The number of mistakes in $S$ is at least $|\cup_{s' \in E} R(S,x_{s'},0)|$. Moreover, $R(S,x,0) \cap R(S,x',0) = \emptyset$ for $x \neq x'$. Lastly, $s' \in E$ implies that $|R(S,x_{s'},0)| \geq d$. Therefore, the number of mistakes in $S$ is at least $\sum_{s' \in E} |R(S,x_{s'},0)| + |R(S,x_s,0)| \geq kd + d = (k+1)d$.
\end{proof}

Next, we prove that as long as no non-exception explanation has been used by $\cA$ at least $d$ times for a true label of $0$, the adversary can set the next label arbitrarily.

\begin{lemma} \label{lem:dobervationsneeded}
Fix $k$, let $L\geq 2$ and $t\in [L]$. Let $S =S_t$ be a sequence generated by an interaction between a learner $\mathcal{A}$ and the adversary $V$ over $t$ rounds, with history $H_{\bullet}$ and exception rounds $E=E_t$. If for all elements $x_s \in S^{\setminus E}_{\X} $ we have $|R(S^{\setminus E},x_s,0)| < d$, then for each $y \in \{0,1\}$, there is a teacher $T_y = (f_y, \psi_y) \in \T_{S^{\setminus E}}$ such that $f_y(x_{t+1}) =y$. \end{lemma}

\begin{proof}
Assume that the conditions of the lemma hold. 
The goal is to show that for any $y^* \in \{0,1\}$ there exists a teacher $T^* = (f^*,\psi^*) \in \cT_{S^{\setminus E}}$ such that $f^*(x_{t+1}) = y^*$.
As argued above, $\cT_{S^{\setminus E}}$ is non-empty. In addition, by the definition of the adversary, $S_\cX$ includes only examples from $\cX_i$. Also, $\cF_i$ includes all labelings of $\cX_i$. Thus, we can find an $f^*$ that is consistent with any labeling of $x_1,\ldots,x_{t+1}$. Set $f^* \in \cF_i$ such that for all $r \in [t] \setminus E$, $f^*(x_r) = y_r$, and for $r \in E$, $f^*(x_r) = 0$. In addition, $f^*(x_{t+1}) = y^*$. 

We have left to show that there exists some $\psi^* \in \Psi_{f^*}$ such that 
for all $r \in [t] \setminus E$ and $\phi_r \neq \bot$, we have $\psi^*(x_r,\hat{x}_r) = \phi_r$. From the definition of $\Psi_{f^*}$, it suffices to find a function $\bar{a}: \X_i \times \{0,1\} \to \mathbb{F}^d_{p_i}$ such that $\psi^* = \psi_{\bar{a}}^{f^*}$ satisfies the requirement, while also $\forall \hat{x} \in \cX_i, f^*(\hat{x}) = 1 \Rightarrow \bar{a}(\hat{x},0)(0)=b_i(f^*)$.

Consider a round $r$ with $\phi_r \neq \bot$. If $\hat{x}_r$ appears in $\thist$,
then $\hat{x}_r \notin \cX_i$, therefore $\phi_r = \one[x_r]$ regardless of
the choice of teacher. Now, consider a round $r \in [t] \setminus E$ in which $\hat{x}_r \in \cX_i$, and
observe that the value of $\bar{a}$ affects $\psi_{\bar{a}}^{f^*}(x_r,\hat{x}_r)$ only via the value of $\bar{a}(\hat{x}_r, f^*(x_r)) = \bar{a}(\hat{x}_r, y_r)$. Therefore, we can set the value of $\bar{a}(\hat{x}, y)$ for each
$(\hat{x},y) \in \cX_i \times \{0,1\}$ separately, by making sure $\psi_{\bar{a}}^{f^*}(x_r,\hat{x}_r) =\phi_r$ for rounds $r \in R(S^{\setminus E}, \hat{x}, y)$. Note that all $\hat{x}$ for which $R(S^{\setminus E}, \hat{x}, y)$ is non-empty for some $y \in \{0,1\}$ are equal to $x_s$ for some $s \in [t]$.

To set $\bar{a}(x_s, y)$ for non empty $R(S^{\setminus E}, x_s,y)$, we distinguish between two cases:
\begin{itemize}
\item 
$s \in [t]\setminus E$ and $y = 0$. In this case, since $s \notin E$, by our assumption, $|R(S^{\setminus E}, x_s, 0)| < d$. Thus, by Lagrange's interpolation theorem, there exist coefficients $\bar{c} \in \bF_{p_i}^d$ such that $\bar{c}(0) = P'[\bar{c}](0) = b_i(f^*)$, and for all rounds $r\in R(S^{\setminus E}, x_s, 0)$, $P[\bar{c}](x_r)= x'_r$. Set $\bar{a}(x_s, 0) := \bar{c}$. Then for all such rounds $r$, 
\begin{align*}
  \psi_{\bar{a}}^{f^*}(x_r, \hat{x}_r) &= \psi_{\bar{a}}^{f^*}(x_r, x_s) = \one[x_r,P[\bar{a}(x_s,f^*(x_r))](x_r)\,]= \one[x_r,P[\bar{a}(x_s,0)](x_r)\,] = \phi_r.
\end{align*}
\item Otherwise ($s \in E$ or $y \neq 0$).    From the definition of the adversary $V$, there exists some teacher $T' = (f',\psi^{f'}_{\bar{a}'}) \in \T$ such that for all rounds $r \in [t] \setminus E$, if $\phi_r \neq \bot$ then $\phi_r = \psi^{f'}_{\bar{a}'}(x_r, \hat{x}_r) = \one[x_r, x'_r]$ for some $x'_r \in  [2^{i+1}+1, 2^{i+1}+p_i]$ and $f'(x_r) = y_r$. Set $\bar{a}(x_s, y) := \bar{a}'(x_s, y)$.
Then
\begin{align*}
  \psi^{f^*}_{\bar{a}}(x_r, x_s) &= \one[x_r,P[\bar{a}(x_s,f^*(x_r))](x_r)\,] = \one[x_r,P[\bar{a}(x_s,y)](x_r)\,]\\
  &= \one[x_r,P[\bar{a}'(x_s,y)](x_r)\,] = \one[x_r,P[\bar{a}'(x_s,f'(x_r))](x_r)\,]  = \psi^{f'}_{\bar{a}'}(x_r, x_s) = \phi_r.
\end{align*}
\end{itemize}
Lastly, for any $\hat{x}$ such that $R(S^{\setminus E}, \hat{x}, y)$ is empty for both $y \in \{0,1\}$, set $\bar{a}(\hat{x},0)(0) := b_i(f^*)$ and set all other values arbitrarily. The resulting $\psi^* = \psi^{f^*}_{\bar{a}}$ outputs $\phi_r$ for $(x_r,\hat{x}_r)$ for every $r \in [t] \setminus E$. 

To verify that $\forall \hat{x} \in \cX_i, f^*(\hat{x}) = 1 \Rightarrow \bar{a}(\hat{x},0)(0)=b_i(f^*)$, first observe that if $\hat{x} = x_s$ for some $s \in E$, then by the choice of $f^*$, $f^*(\hat{x})  = 0$, thus the condition does not hold. In all other cases, $\bar{a}(x_s, 0)(0) = b_i(f^*)$.  The condition thus holds for $\bar{a}$, hence $\psi_{\bar{a}}^{f^*} \in \Psi_f$ and the teacher $T = (f^*, \psi_{\bar{a}}^{f^*}) \cT_{S ^{\setminus E}}$ is consistent with all non-exception rounds $[t] \setminus E$. 
\end{proof}

We can now prove the main lower bound.

\begin{proof}[of \lemref{klower}]
  By \lemref{kd} and \lemref{dobervationsneeded}  in every round $t$ until which the learner has made $(k+1)d-1$ or fewer mistakes, for each $y \in \{0,1\}$, there is a teacher $T_y = (f_y, \psi_y) \in \T_{S_t^{\setminus E_t}}$ such that $f_y(x_{t+1}) =y$.
  In such a round $t+1$, by the definition of the adversary $V$, $\cA$ makes a mistake whenever it selects the more likely label according to its own distribution in that round. The probability that this happens in any specific round is at least $1/2$. Let $n := (k+1)d$. If the number of such rounds is at least $n$ then the number of mistakes is also at least $n$. The probability that there are fewer than $n$ such rounds is at most $P_{n,L} := \sum_{j=0}^{n-1}\binom{L}{j} \left(\frac{1}{2}\right)^L$. Therefore,
  \[
    M^{L}(\mathcal{A} ,V,H) \geq n \cdot (1-P_{n,L}).
  \]

  We upper bound $P_{n,L}$ using the standard binomial upper bound \citep[see, e.g.,][Lemma A.5]{SB14}
  \[
    P_{n,L}= 2^{-L} \sum_{j=0}^{n-1}\binom{L}{j} \leq 2^{-L} \left(\frac{eL}{n-1}\right)^{(n-1)}.
    \]

For $L \geq 4(k+1)d > 4(n-1)$ we can further simplify

\begin{align*}
  P_{n,L} &\leq 2^{-4(n-1)}2^{-(L-4(n-1))}e^{n-1} \left(\frac{L}{4(n-1)}\right)^{n-1}4^{n-1}\\
  &= (4e/2^4)^{n-1}\cdot 2^{-(L-4(n-1))}\left(\frac{L}{4(n-1)}\right)^{n-1}\\
          &\leq \left(\frac{3}{4}\right)^{n-1}\cdot (L^{n-1}2^{-L}) \cdot (4(n-1))^{n-1}2^{-4(n-1)}.
\end{align*}
Since $g(x) := \frac{x^a}{2^x}$ is a monotonously decreasing function in $x \geq 1$ for every $a\geq 0$ and $L \geq 4(n-1)$, we have $L^{n-1}2^{-L} \leq (4(n-1))^{n-1}2^{-4(n-1)}$. Therefore, $P_{n,L} \leq (3/4)^{n-1}$. 
It follows that
\[
  M^{L}(\mathcal{A} ,V,H) \geq n - n(3/4)^{n-1}.
\]
Since $x (3/4)^{x-1} < 2$ for every $x \geq 1$, we have $n (3/4)^{n-1} < 2$. Therefore, 
\[
  M^{L}(\mathcal{A} ,V,H) \geq n - 1 = (k+1)d - 1.
\]
This completes the proof.
\end{proof}